\documentclass[11pt]{article}

\usepackage[preprint]{acl}

\usepackage{amsmath}
\usepackage{amssymb}
\usepackage{amsfonts}
\usepackage{mathtools}
\usepackage{amsthm}
\usepackage{bbm}

\newtheorem{theorem}{Theorem}
\newtheorem{lemma}{Lemma}

\newtheorem{definition}{Definition}

\newtheorem{assumption}{Assumption}

\usepackage{algorithm}
\usepackage{algpseudocode}
\usepackage{array}
\usepackage{booktabs}
\usepackage{multirow}
\usepackage{graphicx}
\usepackage{enumitem}
\usepackage{xspace}
\usepackage{xcolor}
\usepackage[table]{xcolor}

\usepackage{tabularx}
\usepackage{array}

\usepackage{times}
\usepackage{latexsym}

\newcommand{\sig}[1]{\textsuperscript{#1}}

\usepackage[T1]{fontenc}
\usepackage[utf8]{inputenc}

\usepackage{microtype}

\usepackage{inconsolata}

\usepackage{graphicx}

\usepackage{tcolorbox}
\tcbuselibrary{breakable,skins}
\usepackage{xcolor}
\usepackage{booktabs}

\newtcolorbox{roundbox}[2]{
  colback=blue!3,
  colframe=blue!40,
  fonttitle=\bfseries,
  title=#1,
  breakable,
  enhanced,
  sharp corners,
  boxrule=0.6pt,
  left=6pt,
  right=6pt,
  top=4pt,
  bottom=4pt
}

\newtcolorbox{prompt}{
    enhanced,
    breakable,
    colback=gray!10,
    colframe=gray!50,
    boxrule=0.5pt,
    arc=2pt,
    left=5pt, right=5pt, top=5pt, bottom=5pt,
    before skip=0pt, after skip=0pt,
    fontupper=\small\ttfamily,
    before upper={\sloppy\emergencystretch=3em\hbadness=10000\relax},
}

\newcolumntype{A}{>{\columncolor{blue!6}}c}

\title{PEAR: Permutation-Equivariant Adaptive Routing Multi-Agent Debate}

 \author{%
     Yang Feng\\
     University of Edinburgh\\
     \texttt{Y.Feng-85@sms.ed.ac.uk}
     \And
     Ziwei Xu \\
     FAR.AI \\
     \texttt{ziwei.xu@u.nus.edu}\\
     \AND
     \textbf{Xia Hu} \\
     Shanghai AI Laboratory \\
     \texttt{huxia@pjlab.org.cn}
     \And
     Fengxiang He \\
     University of Edinburgh \\
     \texttt{fhe@ed.ac.uk} 
 }

\begin{document}
\maketitle
\begin{abstract}
% Multi-agent debate can improve the reliability of large language models (LLMs) by allowing agents to critique and revise one another's reasoning. However, fixed communication structures can introduce persistent positional advantages, amplify unreliable agents, and make outcomes sensitive to arbitrary role assignments. 
Multi-agent debate improves the reliability of large language models (LLMs) through iterative peer critiques. However, fixed topologies often introduce persistent positional biases, amplify unreliable agents, and cause high sensitivity to role assignments.
We introduce \textit{Permutation-Equivariant Adaptive Routing Multi-Agent Debate (PEAR)}, an inference-time train-free protocol that dynamically reconfigures communication roles and sparse topologies across consecutive debate rounds.
By strategically switching agent-to-role assignments based on evolving agent states, PEAR prevents any agent from permanently occupying a privileged network position or distributes influence more evenly across the debate. We theoretically characterize PEAR as an equivariant sparse router: it preserves accuracy under agent relabeling while reducing routing complexity and improving generalization. Comprehensive empirical evaluations across four reasoning benchmarks and six diverse LLM backbones demonstrate PEAR significantly improves average accuracy over the strongest debate baselines. The code is available at \url{https://github.com/EVIEHub/PEAR}.
%The code is available at \url{https://anonymous.4open.science/r/PEAR-C8AB/README.md}.
%Multi-LLM debate is a promising approach to improve language-model reliability by letting multiple agents critique and refine each other’s reasoning instead of trusting a single run. However, the performance can hinge on fixed communication topologies, early-message anchoring, dominant-speaker effects, and static role assignments that create persistent structural advantages. We address these issues with Itinerant Multi-Agent Debate (IMAD), a protocol that randomly shuffles the inter-agent network topology across rounds, continually changing who interacts with whom and how influence propagates. We formalize IMAD as a stochastic process over communication graphs and prove the shuffle both realizes topology symmetry and acts as a projection operator from the original hypothesis space to its permutation-equivariant counterpart. We show this projection is a concentration operator in terms of covering number, effectively shrinking hypothesis complexity and yielding better generalisability; these underpins practical outcomes: reduced positional bias, lower sensitivity to agent ordering and roles, and improved robustness to noisy or low-quality agents by dispersing their influence. We empirically validate IMAD. %on GSM8K, MATH, MMLU, TruthfulQA, and HotpotQA, where IMAD consistently improves accuracy and calibration while substantially reducing seed- and permutation-induced instability over the state-of-the-art methods.
\end{abstract}

\section{Introduction}
\label{sec:introduction}
% Multi-agent debate has emerged as a promising inference-time paradigm for improving the reliability of large language models (LLMs). Rather than relying on a single response, debate protocols instantiate multiple agents that propose, critique, and revise candidate solutions through iterative interaction. Such deliberation can expose inconsistencies, surface alternative reasoning paths, and help correct errors before a final answer is selected.
Multi-agent debate (MAD) has emerged as a promising inference-time paradigm for enhancing the reasoning capabilities of large language models (LLMs)~\cite{chen2024reconcile:mad-1,du2024multiagentdebate:mad-2}, in which multiple agents independently generate responses and iteratively revise their answers through structured peer critique. By aggregating diverse reasoning paths across agents and rounds, MAD has demonstrated consistent gains on knowledge-intensive and mathematical reasoning benchmarks~\cite{liang2024encouraging:mad-3}.

However, MAD outcomes depend not only on agent capabilities but also on the communication structure through which critiques flow. Most existing protocols use a fixed topology, such as a chain, star, or ring, across all rounds~\cite{liang2024encouraging:mad-3}, which can introduce persistent positional advantages, amplify unreliable agents, and make outcomes sensitive to arbitrary role assignments. These structures create two related failure modes. First, they promote \textit{information homogenization}: agents repeatedly exchange critiques within the same predetermined neighborhoods, so errors can be reinforced rather than challenged. Even when diverse viewpoints are present~\cite{zhu2026demystifying}, a fixed topology may prevent them from reaching the agents who would benefit most. Second, they enable \textit{error cascades}: hubs, early speakers, or otherwise privileged agents can exert disproportionate influence, anchoring the group on incorrect conclusions when their own answers are wrong~\cite{zhang2026key,tian2026multi}. Prior work identifies these symptoms, but leaves the underlying routing structure largely fixed.

% To address these limitations, we introduce {Itinerant Multi-Agent Debate (IMAD)}, a protocol family that varies agent-role assignments across debate rounds. The simplest member of this family, \emph{uniform IMAD}, fixes a base role graph but samples a fresh random assignment of agents to roles in each round. This converts persistent positional privilege into temporary and distributed influence: an agent may occupy a central or high-information role in one round, but is unlikely to retain that role throughout the entire debate.
To address these limitations, we introduce {Permutation-Equivariant Adaptive Routing Multi-Agent Debate (PEAR)}, an inference-time train-free framework that replaces fixed communication topologies with state-aware adaptive routing. Rather than fixing the communication graph across rounds, PEAR dynamically selects a sparse communication topology at each debate round based on the current state of all agents, including their answers, self-reported confidence scores, and accumulated influence from prior rounds. Figure~\ref{fig:pipeline} illustrates the pipeline.

\begin{figure*}[t]
    \centering
    \includegraphics[width=\linewidth]{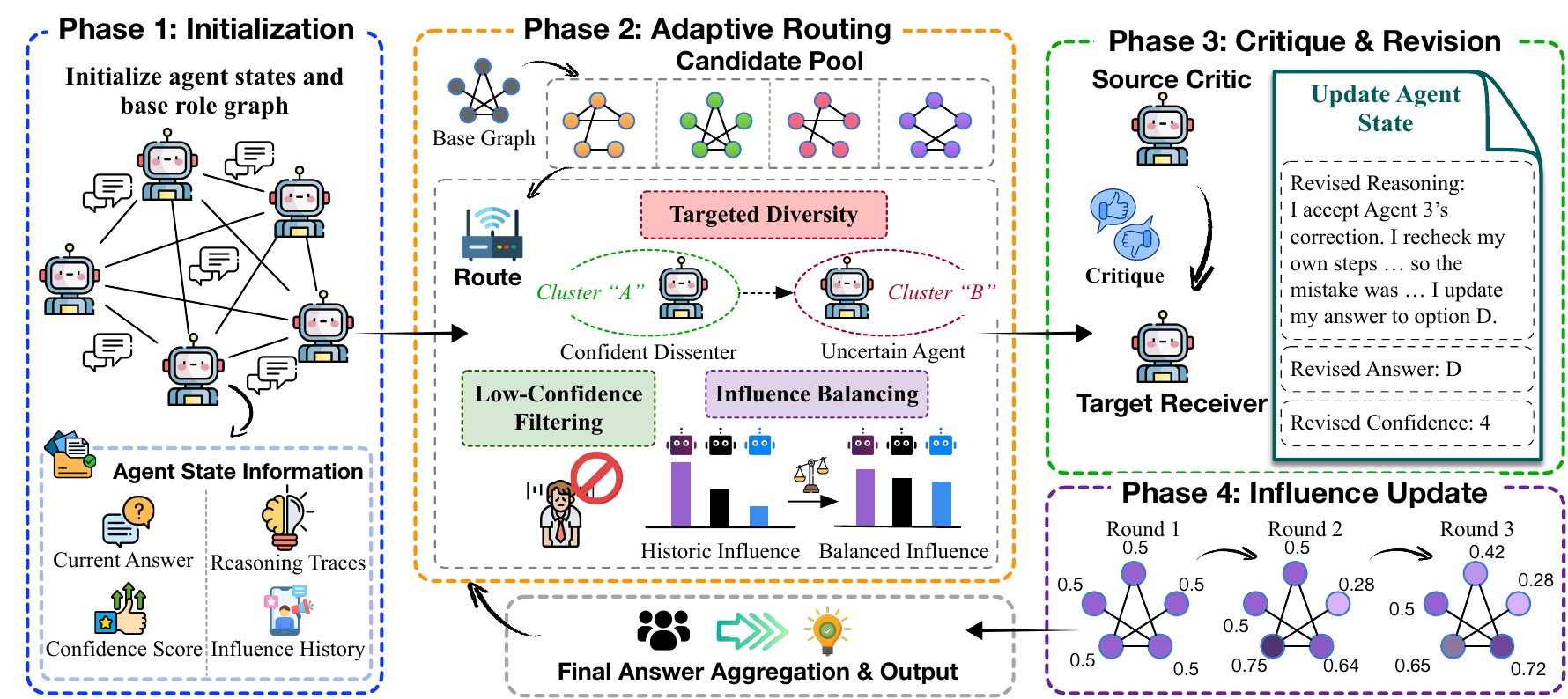}
    \caption{Overview of PEAR. Given a task instance, multiple agents produce independent initial responses. At every debate round, the router selects the communication topology that maximizes a composite score over three state-aware components: targeted diversity, influence balancing, and low-confidence filtering. Sources then critique their assigned targets, targets revise their answers, and the router updates each agent's accumulated influence from the adoption pattern. After several rounds, the final answer is obtained by majority vote.}
    \label{fig:pipeline}
\end{figure*}

This routing decision is governed by a composite objective with three components. \textit{Targeted Diversity} prioritizes edges from agents with high confidence and differing answers toward agents with low confidence, ensuring that reliable dissenting viewpoints are routed to the agents most in need of external correction. \textit{Influence Balancing} penalizes sources that have accumulated disproportionate historical influence, preventing any single agent from dominating the information flow regardless of whether its answer is correct. \textit{Low-Confidence Filtering} suppresses the propagation of critiques from agents with low self-reported confidence, reducing the transmission of unreliable signals through the debate graph. Together, these three components operationalize a principled routing policy that actively promotes wrong-to-right corrections while attenuating the structural pathways through which errors cascade. 

We prove that PEAR can be viewed as an agent-permutation equivariant router: relabeling agents relabels the communication graph but leaves the final-answer distribution unchanged. This follows from its search over sparse relabelings of a fixed base graph and routing scores based only on label-invariant features such as disagreement, confidence, and accumulated influence. We further prove that equivariant routing preserves expected accuracy under exchangeable agent populations and does not increase covering complexity, yielding a uniform generalization guarantee. These results suggest that PEAR improves debate by combining adaptivity with the natural symmetries of multi-agent reasoning.

%\paragraph{Experiments.}
We evaluate PEAR on four benchmarks, MMLU-Pro~\cite{wang2024mmlu:mmlu-pro}, TruthfulQA~\cite{lin2022truthfulqa:truthfulqa}, GSM8K~\cite{DBLP:journals/corr/abs-2110-14168:gsm8k}, and MATH-500~\cite{hendrycks2021measuring:math500}, spanning knowledge-intensive reasoning, factual question answering and competition-level mathematics. Experiments are conducted across six instruction-tuned models of varying scale and family, including Gemma-3-12B~\cite{gemmateam2025gemma3technicalreport}, Llama-3.1-8B~\cite{DBLP:journals/corr/abs-2407-21783:llama}, Qwen2.5-14B, and Qwen-3-30B-A3B~\cite{DBLP:journals/corr/abs-2505-09388:qwen3}, together with two closed-source models, GPT-5.4-nano~\cite{openai2026gpt54} and Claude-Haiku-4.5~\cite{anthropic2025haiku45}. PEAR consistently outperforms all fixed-topology debate variants across all settings, achieving gains of up to 9.0\% points over the strongest baseline, with ablation studies confirming the independent contribution of each routing component. 

To the best of our knowledge, this is the first paper in the literature that realizes equivariant communication in multi-agent debate, which significantly improves the performance of LLM reasoning.

%Code is at xxx.

% Our main contributions are as follows:
% \begin{itemize}[noitemsep,topsep=2pt]
%     \item We identify information homogenization and error cascades as two structural failure modes of fixed-topology MAD, and propose a unified information-routing perspective to address them.
%     \item We introduce PEAR, an inference-time framework that dynamically selects communication topologies based on agent confidence, answer diversity, and accumulated influence, requiring no model fine-tuning or ground-truth supervision.
%     \item We provide comprehensive empirical evaluation across four benchmarks and five models, with trajectory-level diagnostics that directly measure the corrective and corrupting dynamics of debate.
% \end{itemize}
% Our main contributions are as follows. We identify information homogenization and error cascades as two structural failure modes of fixed-topology MAD, and propose a unified information-routing perspective to address them. Building on this perspective, we introduce PEAR, an inference-time framework that dynamically selects communication topologies based on agent confidence, answer diversity, and accumulated influence, requiring no model fine-tuning or ground-truth supervision. Finally, we provide a comprehensive empirical evaluation across four benchmarks and five models, with trajectory-level diagnostics that directly measure the corrective and corrupting dynamics of debate.

\section{Related Works}
\label{sec:related-work}

\paragraph{Multi-agent debate (MAD).}
%MAD has emerged as an effective inference-time paradigm for improving the reasoning capability of LLMs. 
Early work \cite{du2024multiagentdebate:mad-2,chen2024reconcile:mad-1,liang2024encouraging:mad-3} shows that multiple agents iteratively critiquing and revising each other's answers can outperform single-agent inference on knowledge-intensive and reasoning tasks. 
Subsequently, \citet{liu2024groupdebate} partition agents into subgroups to reduce token cost, while \citet{lin2025enhancing,taubenfeld2025confidence,fu2025deep} incorporate self-reported confidence to weight or filter agents' contributions. More recently, \citet{zhu2026demystifying} identify viewpoint homogeneity as a primary cause of debate stagnation; \citet{zhang2026key} show that positionally dominant agents disproportionately determine the final outcome; \citet{tian2026multi} reveal that erroneous content accumulated across rounds can persistently mislead downstream agents; and \citet{nguyen2026hear} preserve diversity by selectively retaining historically dissenting messages. 
%These works %diagnose failure modes induced by \emph{who} speaks and \emph{what} is remembered, but 
%largely retain a fixed communication structure across rounds. 
%AR-MAD %treats the communication topology as a first-class control variable and 
%dynamically routes critiques, %based on the instantaneous system state, thereby 
%mitigating the induced information homogenization and error propagation. %by fixed interaction structures.

\paragraph{Communication topology.}
%A parallel line of research investigates how communication structure shapes multi-agent performance. In LLM-based systems, 
~\citet{li2024improving} demonstrate that sparse topologies can match or exceed fully connected debate at substantially lower cost. %, motivating the use of bandwidth-limited graphs. 
Subsequently, %methods learn or design the topology itself: 
~\citet{zhang2024g:g-designer} propose G-Designer, which uses graph neural networks to architect task-specific communication graphs; ~\citet{shen2025assemble} autoregressively generate agent topologies; and ~\citet{li2025adaptive} adaptively prune edges to balance accuracy and efficiency. Beyond LLM debate, ~\citet{hu2024learning} learn communication graphs end-to-end, ~\citet{sun2024t2mac} perform targeted message selection, ~\citet{guan2024efficient} aggregate information via self-supervision, and ~\citet{10720863} optimize communication at the semantic level. 
These approaches typically %optimize a %\emph{single} 
%topology, either offline or once per task, 
rely on training signals or learned controllers. In contrast, PEAR adapts the topology purely at inference.  %using the evolving debate state to route information 
%without any model fine-tuning or ground-truth supervision.
\section{Preliminaries}
\label{sec:preliminaries}

%\paragraph{Task and agents.}
We consider a task instance $x\in\mathcal{X}$ with ground-truth label $y^\star\in\mathcal{Y}$ drawn from an unknown distribution $\mathcal{D}$. A debate is conducted by $n\geq 2$ agents indexed by $[n]=\{1,\ldots,n\}$. Each agent $i$ is a stochastic inference-time policy $\pi_{\theta_i}(\cdot\mid x,o_i)$ that maps the task instance and a local observation $o_i$ to a distribution over outputs.
%\paragraph{Per-agent state and debate state.}
At the end of each round $r$, every agent $i$ exposes a structured tuple
%\begin{equation*}
    \(\xi_i^{(r)}=\big(y_i^{(r)},\,c_i^{(r)},\,r_i^{(r)}\big)\),
%\end{equation*}
where $y_i^{(r)}\in\mathcal{Y}$ is the current answer, $c_i^{(r)}\in\{1,\ldots,5\}$ is the self-reported confidence on an ordinal scale, and $r_i^{(r)}$ is the accompanying chain-of-thought reasoning. In parallel, the router maintains an accumulated \emph{influence} statistic $\rho_i^{(r)}\in[0,1]$ that summarizes how often agent $i$'s prior critiques have been adopted by its downstream targets; its update rule is given in Section~\ref{sec:method}. The \emph{debate state} observable to the router at the start of round $r$ is
%\begin{equation*}
    \(s_r=\big\{\big(\xi_i^{(r-1)},\rho_i^{(r-1)}\big)\big\}_{i\in[n]}\).
%\end{equation*}

\paragraph{Role graph.}
%We distinguish \emph{roles} from \emph{agents}. 
Let
%\begin{equation*}
    $G_0=(V_0,E_0)$, $V_0=[n]$,
%\end{equation*}
be a directed base role graph that fixes a sparse communication pattern. 
An edge $(u\!\to\!v)\in E_0$ permits the agent occupying role $v$ to observe messages produced by the agent occupying role $u$. A role-to-agent assignment is a bijection $\pi\in S_n$, where $\pi(u)$ is the agent occupying role $u$. The induced agent-level communication graph is
\begin{align*}
    G(\pi) &= ([n],\,E(\pi)),\\
    E(\pi) &= \{(\pi(u)\!\to\!\pi(v)):(u\!\to\!v)\in E_0\}.
\end{align*}
By construction, every $G(\pi)$ preserves the per-agent in-degree of $G_0$, so the assignment $\pi$ only relabels which concrete agent occupies each role and yields a candidate sparse communication graph rather than an arbitrary subset of edges.
Since every role has in-degree $k$, the base graph has
\(
    m := |E_0| = nk
\)
directed edges. Every candidate graph $G(\pi)$ therefore satisfies
$|E(\pi)|=m$.

\paragraph{Round-based debate.}
A debate proceeds for $R$ rounds. At the start of round $r$, the round-$r$ communication graph $G^{(r)}=G(\pi_r)$ is selected as a state-aware function of $s_r$; the specific selection rule is given in Section~\ref{sec:method}. After $R$ rounds, the protocol returns $\hat{y}=\mathrm{Agg}\big(\{y_i^{(R)}\}_{i\in[n]}\big)$, where $\mathrm{Agg}$ is the aggregation rule (we use majority vote).
\section{Permutation-Equivariant Adaptive Routing Multi-Agent Debate}
\label{sec:method}
 
%Adaptive Routing Multi-Agent Debate (AR-MAD) is an inference-time protocol that casts multi-agent debate as a state-aware information-routing problem. Given a task instance $x$ and $n$ agents, 
PEAR proceeds in four phases: Initialization, Adaptive Routing, Critique \& Revision, and Influence Update.
%\textbf{Initialization}, in which the base communication graph is fixed and agents respond independently; \textbf{Adaptive Routing}, in which the router selects a sparse communication topology by scoring candidates with three components, \emph{Targeted Diversity}, \emph{Influence Balancing}, and \emph{Low-Confidence Filtering}; \textbf{Critique \& Revision}, in which each source agent produces a targeted critique and the corresponding target revises its answer; and \textbf{Influence Update}, in which the router refreshes per-agent influence statistics from the adoption pattern just observed. 
The latter three are repeated for $R$ rounds. Figure~\ref{fig:pipeline} provides a schematic overview. %of the protocol. %and Algorithm~\ref{alg:armad} summarizes the full procedure.

\subsection{Initialization}
\label{sec:method:initial}
 
\paragraph{Base role graph.}
PEAR instantiates the base role graph $G_0$ from Section~\ref{sec:preliminaries} as a sparse $k$-regular template in which every role has in-degree $k$. The $k$-regular template is preferred over a clique for three reasons: (i) it caps the per-round critique budget at $nk$ rather than $n(n{-}1)$ edges; (ii) it keeps the per-agent input bounded, mitigating the noise introduced by aggregating critiques from every other agent; and (iii) it admits a tractable space of candidate topologies $\{G(\pi):\pi\in S_n\}$ obtained by relabeling roles, so the router can search among meaningfully different communication graphs while preserving each agent's input budget at $k$.

\paragraph{Initial responses.}
Before any communication occurs, each agent $i\in[n]$ independently produces an initial response $\xi_i^{(0)}\sim\pi_{\theta_i}(\cdot\mid x)$, using the structured tuple defined in Section~\ref{sec:preliminaries}. Agents are instructed to reserve the low end of the confidence scale for guesses or for cases in which competing alternatives cannot be ruled out, and to reserve the highest score for fully verified solutions. The router-side influence statistic is initialized to $\rho_i^{(0)}=0$ for every agent.

\subsection{Adaptive Routing}
\label{sec:method:routing}

\paragraph{Candidate pool.}
At the start of round $r$, the router constructs a finite candidate pool
%\begin{equation*}
\(    \mathcal{C}_r=\{\pi_r^{(1)},\ldots,\pi_r^{(M)}\}\subseteq S_n\), \(M=|\mathcal{C}_r|\),
%\end{equation*}
by enumerating assignments that yield distinct edge sets $E(\pi_r^{(m)})$. The size $M$ is upper-bounded by a configurable budget $M_{\max}$. The candidate pool merely delimits the search space over which the state-aware score is computed; it is not the routing decision itself. 
For a candidate topology $G(\pi)$ with edge set $E=E(\pi)$, the router computes a composite score $S(E\mid s_r)$ from three state-aware components, each normalized so that the relative magnitudes of the weights $\alpha_T,\alpha_I,\alpha_L$ remain interpretable.
 
\paragraph{Targeted diversity.}
This component rewards edges whose source is confident, whose target is uncertain, and whose two endpoints currently hold different answers. Concretely, for an edge $(s\!\to\!t)\in E$,
\begin{align}
    T(s,t)=
    & \mathbf{1}\big[y_s^{(r-1)}\!\neq\! y_t^{(r-1)}\big]
    \mathbf{1}\big[c_s^{(r-1)}\!\geq\!\tau_{\mathrm{src}}\big]\nonumber\\
    &
    \mathbf{1}\big[c_t^{(r-1)}\!\leq\!\tau_{\mathrm{tgt}}\big],
%\end{}
\label{eq:targeted-diversity}
\end{align}
where $\tau_{\mathrm{src}}$ and $\tau_{\mathrm{tgt}}$ are confidence cutoffs for the source and target, respectively. The normalized rate over a candidate edge set is
\begin{equation}
    %\widetilde{T}(E)=\frac{1}{|E|}\sum_{(s,t)\in E}T(s,t)\in[0,1].
    \widetilde T(E)
=
\frac{1}{m}
\sum_{(s,t)\in E} T(s,t)
\in[0,1].
    \label{eq:targeted-diversity-rate}
\end{equation}
Equation~\eqref{eq:targeted-diversity} is asymmetric by design: it favors directed dissent from high-confidence agents toward low-confidence agents, in line with the empirical observation that reliable, opinion-divergent sources are the most likely to convert uncertain targets from wrong to right.

\paragraph{Influence balancing.}
Let $\rho_s^{(r-1)}\in[0,1]$ denote the accumulated influence of agent $s$ entering round $r$, and let $d^{\mathrm{out}}_s(E)=|\{t:(s,t)\in E\}|$ be its out-degree under the candidate edge set. The influence-balancing component penalizes routing structures that allocate additional out-degree to already dominant agents:
\begin{equation}
    %\widetilde{I}(E)=\frac{1}{|E|}\sum_{s\in[n]}\rho_s^{(r-1)}\,d^{\mathrm{out}}_s(E)\in[0,1].
        \widetilde{I}(E)=\frac{1}{m}\sum_{s\in[n]}\rho_s^{(r-1)}\,d^{\mathrm{out}}_s(E)\in[0,1].
    \label{eq:influence-balancing}
\end{equation}
Here $\rho_s^{(r-1)}\in[0,1]$ by induction from the EMA update in
Eq.~(8), and $\sum_s d_s^{\rm out}(E)=m$, so
$\sum_s \rho_s^{(r-1)}d_s^{\rm out}(E)\le m$.
This term provides a structural defense against error cascades: an agent that has previously persuaded many targets is prevented from being amplified further, regardless of whether its current answer is correct.

\algrenewcommand\algorithmicindent{1em}

\begin{algorithm}[t]
\caption{PEAR}%: Adaptive Routing Multi-Agent Debate}
\label{alg:armad}
\begin{algorithmic}[1]
\small
%\Require Input $x$; agents $\{\pi_{\theta_i}\}_{i=1}^{n}$; base graph $G_0$; rounds $R$; weights $(\alpha_T,\alpha_I,\alpha_L)$; thresholds $(\tau_{\mathrm{src}},\tau_{\mathrm{tgt}},\tau_{\mathrm{low}})$; smoothing $\beta$
\Require Input $x$; agents $\{\pi_{\theta_i}\}_{i=1}^n$; base graph
$G_0$; rounds $R$; weights $(\alpha_T,\alpha_I,\alpha_L)$; thresholds
$(\tau_{\rm src},\tau_{\rm tgt},\tau_{\rm low})$; smoothing $\beta$;
temperature $\tau$
\State For all $i\in[n]$: $\xi_i^{(0)}\!\sim\!\pi_{\theta_i}(\cdot\mid x)$ and $\rho_i^{(0)}\!\leftarrow\!0$ 
% \hfill\Comment{Initialization}
\For{$r=1,\ldots,R$}
\State $\pi_r \sim Q_r(\cdot\mid s_r)$; $E^{(r)}\leftarrow E(\pi_r)$
\State \Comment{softmax routing; reduces to argmax as $\tau\to0$}
    %\State $\pi_r\leftarrow\operatorname*{arg\,max}_{\pi\in\mathcal{C}_r}S(E(\pi)\!\mid\!s_r)$;\;\,$E^{(r)}\!\leftarrow\!E(\pi_r)$
    %\State \Comment{Adaptive Routing}
    % \For{$(s,t)\!\in\! E^{(r)}$} 
    % \State For $(s,t)\!\in\! E^{(r)}$: $m_{s\to t}^{(r)}\!\sim\!\pi_{\theta_s}(\cdot\mid x,\xi_t^{(r-1)},\xi_s^{(r-1)})$
    \State For $(s,t)\!\in\! E^{(r)}$: $m_{s\to t}^{(r)} \sim \pi_{\theta_s}\big(\cdot \mid x,\xi_s^{(r-1)},\xi_t^{(r-1)}\big)$
    \State \hfill\Comment{Critique}
    % \EndFor
    \State $\xi_t^{(r)}\!\sim\!\pi_{\theta_t}(\cdot\mid x,\xi_t^{(r-1)},\mathcal{I}_t^{(r)})$ for $t\in[n]$ 
    \hfill\Comment{Revision}
    \State $\rho_s^{(r)}\!\leftarrow\!\beta\rho_s^{(r-1)}\!+\!(1\!-\!\beta)a_s^{(r)}$ for $s\in[n]$ \hfill\Comment{Influence}
\EndFor
\State \Return $\hat{y}=\mathrm{Agg}\big(\{y_i^{(R)}\}_{i\in[n]}\big)$
\end{algorithmic}
\end{algorithm}
 
\paragraph{Low-confidence filtering.}
This component suppresses critique edges whose source is unreliable. For each source $s$, define a confidence penalty
\begin{equation}
    L(s)
=
\max\left\{0,\tau_{\rm low}+1-c_s^{(r-1)}\right\},
    \label{eq:lcf}
\end{equation}
which equals zero whenever the source confidence exceeds the threshold $\tau_{\mathrm{low}}$ and grows linearly as confidence drops below it. We take $\tau_{\rm low}$ to be an integer threshold on the ordinal confidence
scale. Since $c_s^{(r-1)}\in\{1,\ldots,5\}$, the penalty is zero exactly when
$c_s^{(r-1)}\ge \tau_{\rm low}+1$, equivalently when
$c_s^{(r-1)}>\tau_{\rm low}$.
The normalized penalty rate is
\begin{equation}
    %\widetilde{L}(E)=\frac{1}{|E|\cdot\tau_{\mathrm{low}}}\sum_{(s,t)\in E}L(s)\in[0,1].
        \widetilde L(E)
=
\frac{1}{m\tau_{\rm low}}
\sum_{(s,t)\in E} L(s)
\in[0,1].
    \label{eq:lcf-rate}
\end{equation}
The denominator $\tau_{\rm low}$ is chosen because the maximum possible
penalty is attained at minimum confidence $c_s=1$, where
$L(s)=\tau_{\rm low}$.
Unlike a confidence-maximizing rule, Equation~\eqref{eq:lcf} only penalizes \emph{low}-confidence sources and never rewards the highest-confidence one, consistent with the observation that confidence is informative as a filter but unreliable as a sole selector.
 
\paragraph{Composite score and topology selection.}
The state-conditioned routing score for candidate $\pi$ is the linear combination
\begin{align}
S\big(E(\pi)\,\big|\,s_r\big)
&=
\alpha_T\,\widetilde{T}\big(E(\pi)\big) \\
&\quad
-\alpha_I\,\widetilde{I}\big(E(\pi)\big)
-\alpha_L\,\widetilde{L}\big(E(\pi)\big),\nonumber
% \label{eq:composite-score}
\end{align}
with non-negative weights $(\alpha_T,\alpha_I,\alpha_L)$ that are held fixed across rounds. Given the composite scores $\{S(E(\pi)\mid s_r):\pi\in\mathcal{C}_r\}$, the router selects the round-$r$ assignment via the state-conditional distribution
\begin{align}
    Q_r(\pi\mid s_r)
    &=
    \frac{\exp\!\big(S(E(\pi)\mid s_r)/\tau\big)}
         {\sum_{\pi'\in\mathcal{C}_r}\exp\!\big(S(E(\pi')\mid s_r)/\tau\big)},\notag\\
    \pi_r &\sim Q_r(\cdot\mid s_r),
\label{eq:softmax-selection}
\end{align}
with temperature $\tau$. Equation~\eqref{eq:softmax-selection} concentrates mass on the highest-scoring candidate while preserving a controllable degree of exploration over near-optimal candidates, and reduces to argmax selection as $\tau\!\to\!0$. The resulting round-$r$ topology is $G^{(r)}=G(\pi_r)$.

\subsection{Critique \& Revision}
\label{sec:method:critique}
 
\paragraph{Critique generation.}
For each directed edge $(s\!\to\!t)\in E^{(r)}$, the source agent reads the target's previous answer $y_t^{(r-1)}$, reasoning $r_t^{(r-1)}$, and confidence $c_t^{(r-1)}$, and emits a targeted critique conditioned on its own current state:
\begin{equation*}
    m_{s\to t}^{(r)}\sim\pi_{\theta_s}\big(\cdot\mid x,\,\xi_t^{(r-1)},\,\xi_s^{(r-1)}\big).
\end{equation*}
The critique format is unconstrained beyond requiring a verdict and a justification.
 
\paragraph{Answer update.}
Each target $t$ then revises its output given the incoming critique set $\mathcal{I}_t^{(r)}=\{m_{s\to t}^{(r)}:(s\!\to\!t)\in E^{(r)}\}$:
\begin{align*}
    \xi_t^{(r)}=&\big(y_t^{(r)},\,c_t^{(r)},\,r_t^{(r)}\big)\\
    \sim&\pi_{\theta_t}\big(\cdot\mid x,\,\xi_t^{(r-1)},\,\mathcal{I}_t^{(r)}\big).
\end{align*}
In the same response, the target emits an \texttt{ACCEPT}/\texttt{REJECT} decision for each incoming critique; these decisions are recorded by the router and consumed in the next phase.
 
\subsection{Influence Update}
\label{sec:method:influence}
 
After all agents have updated, the router computes the round-$r$ adoption rate of each source $s$,
\begin{equation*}
    a_s^{(r)}
    =
    \frac{\big|\{t:(s,t)\in E^{(r)},\;m_{s\to t}^{(r)}\text{ accepted}\}\big|}
         {\max\big(1,\,d^{\mathrm{out}}_s(E^{(r)})\big)},
\end{equation*}
and refreshes the influence statistic with an exponential moving average:
\begin{equation}
    \rho_s^{(r)}=\beta\,\rho_s^{(r-1)}+(1-\beta)\,a_s^{(r)},
    \label{eq:influence-update}
\end{equation}
with smoothing coefficient $\beta\in[0,1)$. The updated influence enters the routing score at round $r{+}1$ through Equation~\eqref{eq:influence-balancing}, closing the loop between past persuasion and future structural privilege.

\section{Theoretical Results}
\label{sec:theory}

%This section presents theoretical guarantees. Detailed proofs are given in Appendix \ref{app:theory-proofs}.

%We give a theoretical account of AR-MAD as an agent-permutation equivariant routing rule. The guiding principle is that agent labels are arbitrary: if the agents are relabeled, the communication pattern and intermediate states should be relabeled in the same way, while the distribution of the final answer remains unchanged. This parallels permutation-equivariant mechanism design, where symmetrization preserves the objective while reducing unnecessary dependence on arbitrary labels~\citep{qin2022benefits}.

\paragraph{Formalization.}

Let $\mathfrak S_n$ be the symmetric group over agents. For
$\sigma\in\mathfrak S_n$, write $\sigma s$ for the debate state obtained by
renaming each agent $i$ as $\sigma(i)$, and write
\(
    \sigma E = \{(\sigma u,\sigma v):(u,v)\in E\}
\)
for the corresponding relabeling of a directed edge set. We assume the candidate
family is closed under relabeling:
\(
    E\in\mathcal C(s)
    \Longleftrightarrow
    \sigma E\in\mathcal C(\sigma s)
\).
This holds for the full role-assignment family induced by a fixed base graph,
since each candidate is obtained by assigning concrete agents to the same set of
roles. Let $m=|E_0|=nk$ denote the number of critique edges in every candidate
graph.

\paragraph{Equivariance.}
PEAR scores each candidate edge set by
\(
    \alpha_T\widetilde T(E\mid s)
    -
    \alpha_I\widetilde I(E\mid s)
    -
    \alpha_L\widetilde L(E\mid s)\),
where $\widetilde T$ rewards confident disagreement toward uncertain targets,
$\widetilde I$ penalizes allocating out-degree to previously influential agents,
and $\widetilde L$ penalizes low-confidence sources. We thus have the following theorem that proves PEAR is an agent-equivariant router. The proof is in Appendix \ref{app:armad-equivariant}. %The softmax router is
%\[
%    Q_\alpha(E\mid s)
%    =
%    \frac{\exp(S_\alpha(E\mid s)/\tau)}
%    {\sum_{E'\in\mathcal C(s)}\exp(S_\alpha(E'\mid s)/\tau)} .
%\]

\begin{theorem}[Agent-permutation equivariance]
\label{thm:armad-equivariant}
Assume the candidate family is closed under agent relabeling. Then, for x
$\sigma\in\mathfrak S_n$,
%\begin{gather*}    
    $S_\alpha(\sigma E\mid \sigma s)=S_\alpha(E\mid s)$, %\\
    $Q_\alpha(\sigma E\mid \sigma s)=Q_\alpha(E\mid s)$.
%\end{gather*}
\end{theorem}

%The theorem proves that AR-MAD is an agent-equivariant router. %In the deterministic limit $\tau\to0$, the same statement holds for argmax routing whenever ties are broken equivariantly, for example uniformly among maximizers.

%formalizes the role of the base graph: it fixes a sparse communication budget, while the role-to-agent assignment searches over permutations of this budget without privileging arbitrary agent labels.

\paragraph{Symmetrized routing.}
For comparison, consider any possibly non-equivariant routing policy over full
routing transcripts. Let
\(
    \Gamma=(E^{(1)},\ldots,E^{(R)})
\)
be a sequence of selected edge sets, and let $q(\Gamma\mid s^{(0)})$ be the
distribution over routing transcripts from initial state $s^{(0)}$. Define
orbit-averaged router
\(
    (\mathcal P q)(\Gamma\mid s^{(0)})
    =
    \frac{1}{n!}
    \sum_{\sigma\in\mathfrak S_n}
    q(\sigma\Gamma\mid \sigma s^{(0)})\),
where
$\sigma\Gamma=(\sigma E^{(1)},\ldots,\sigma E^{(R)})$.
Let $A(\Gamma,s^{(0)})\in[0,1]$ denote the final-answer accuracy induced by
routing $\Gamma$ from initial state $s^{(0)}$.

\begin{theorem}[Accuracy preservation]
\label{thm:accuracy-preservation}
Suppose the distribution of initial debate states is exchangeable, the agent
update kernels are permutation-equivariant, and the final aggregation rule is
permutation-invariant. Then
\[
    \mathbb E_{s^{(0)},\Gamma\sim \mathcal Pq}
    \left[A(\Gamma,s^{(0)})\right]
    =
    \mathbb E_{s^{(0)},\Gamma\sim q}
    \left[A(\Gamma,s^{(0)})\right].
\]
Moreover, $\mathcal Pq$ is agent-equivariant:
\[
    (\mathcal Pq)(\sigma\Gamma\mid \sigma s^{(0)})
    =
    (\mathcal Pq)(\Gamma\mid s^{(0)})\text{, }
    \forall \sigma\in\mathfrak S_n .
\]
\end{theorem}

Thus, under exchangeable agent populations, enforcing equivariance does not
change expected final accuracy; it only removes arbitrary dependence on the agent
names. The proof is in Appendix \ref{app:accuracy-preservation}.

\paragraph{Generalization.}
We next formalize the complexity benefit of equivariant routers. Let
$\mathcal Q$ be a class of one-round routers $q(\cdot\mid s)$, and define
\(
    d(q,q')
    =
    \sup_s
    \left\|
    q(\cdot\mid s)-q'(\cdot\mid s)
    \right\|_1 
\).
Let
\(
    \mathcal P\mathcal Q=\{\mathcal Pq:q\in\mathcal Q\}
\)
be the orbit-projected class. We use covering number to measure the hypothesis complexity; see Definition \ref{def:covering-number} in Appendix \ref{app:complexity}.

\begin{lemma}[Covering number control]
\label{lem:covering-projection}
For every $\epsilon>0$, the covering number satisfies,
\(
    \mathcal N(\mathcal P\mathcal Q,\epsilon,d)
    \le
    \mathcal N(\mathcal Q,\epsilon,d)
\).
\end{lemma}

\begin{theorem}[Generalization bound]
\label{thm:generalization}
Let $z_1,\ldots,z_{N_{\rm ex}}$ be i.i.d. task instances, and let
$F(q;z)\in[0,1]$ be an evaluation functional, such as final-answer accuracy,
one-round correction, or negative routing regret. Suppose $F$ is
$L$-Lipschitz in $q$ under $d$:
\(
    |F(q;z)-F(q';z)|\le Ld(q,q')\),
\(    \forall q,q',z 
\).
Define
\(
    \mathcal L(q)=\mathbb E_z[F(q;z)]\) and
\(    \widehat{\mathcal L}_{N_{\rm ex}}(q)
    =
    \frac{1}{N_{\rm ex}}
    \sum_{\ell=1}^{N_{\rm ex}}F(q;z_\ell)\).
Then, for any $\epsilon>0$ and $\delta\in(0,1)$, with probability at least
$1-\delta$,
%\[
%\sup_{q\in\mathcal P\mathcal Q}
%\left|
%\mathcal L(q)-\widehat{\mathcal L}_{N_{\rm ex}}(q)
%\right|
%\le
%2L\epsilon
%+
%\sqrt{
%\frac{
%\log\left(2\mathcal N(\mathcal P\mathcal Q,\epsilon,d)/\delta\right)
%}{
%2N_{\rm ex}
%}}
%.
%\]
%Consequently, by Lemma~\ref{lem:covering-projection},
\begin{align*}    
%&\sup_{q\in\mathcal P\mathcal Q}
&\left|
\mathcal L(q)-\widehat{\mathcal L}_{N_{\rm ex}}(q)
\right|\\
\le&
2L\epsilon
+
\sqrt{
\frac{
\log\left(2\mathcal N(\mathcal Q,\epsilon,d)/\delta\right)
}{
2N_{\rm ex}
}}
.
\end{align*}
\end{theorem}

Lemma \ref{lem:covering-projection} and Theorem~\ref{thm:generalization} show that restricting to equivariant routers cannot increase the covering complexity of the routing class. PEAR therefore combines a sparse communication budget with an equivariant search space, reducing both per-round debate cost and the effective complexity of the router. Proofs are in Appendix \ref{app:complexity}.

\paragraph{One-round correction and influence-balancing.} More theoretical results on one-round correction and influence-balancing are given in Appendix \ref{app:deferred-theory}.
\section{Experiments}
\label{sec:experiments}

\begin{table*}[t]
\centering
\fontsize{8}{8.5}\selectfont
\setlength{\tabcolsep}{3.5pt}
\renewcommand{\arraystretch}{1.15}

\caption{
Final-answer accuracy across datasets and models evaluated on subset settings.
Results are averaged over 5 random seeds and reported as mean $\pm$ standard deviation.
Best results in each row are bolded.
%AR-MAD consistently outperforms all fixed and random topology baselines across all datasets and model families, demonstrating the advantage of adaptive routing over static communication structures.
}
\label{tab:main-results}

\begin{tabular}{lcccccccc}
\toprule

% & \multicolumn{2}{c}{\textbf{Single-Agent}} &
% \multicolumn{5}{c}{\textbf{Fixed / Random Topologies}} &
% \textbf{Adaptive} \\
% \cmidrule(lr){2-3}
% \cmidrule(lr){4-8}

% \textbf{Dataset} &
\textbf{Model} &
\textbf{CoT} &
\textbf{CoT-SC} &
\textbf{Clique} &
\textbf{Star} &
\textbf{Chain} &
\textbf{Ring} &
\textbf{Random} &
\textbf{PEAR} \\
\midrule

% ===================== MMLU =====================
\rowcolor{gray!10}
\multicolumn{9}{c}{\textbf{MMLU-Pro}} \\
\midrule

 Gemma-3-12B      & 0.375$_{\pm .018}$ & 0.410$_{\pm .016}$ & 0.615$_{\pm .014}$ & 0.605$_{\pm .015}$ & 0.525$_{\pm .022}$ & 0.575$_{\pm .018}$ & 0.585$_{\pm .017}$ & \textbf{\cellcolor{blue!6}0.645$_{\pm .013}$} \\
 Llama-3.1-8B     & 0.415$_{\pm .021}$ & 0.460$_{\pm .020}$ & 0.595$_{\pm .018}$ & 0.560$_{\pm .021}$ & 0.520$_{\pm .019}$ & 0.495$_{\pm .023}$ & 0.540$_{\pm .020}$ & \textbf{\cellcolor{blue!6}0.625$_{\pm .017}$} \\
 Qwen-2.5-14B     & 0.440$_{\pm .017}$ & 0.440$_{\pm .019}$ & 0.540$_{\pm .015}$ & 0.560$_{\pm .018}$ & 0.620$_{\pm .014}$ & 0.600$_{\pm .015}$ & 0.590$_{\pm .016}$ & \textbf{\cellcolor{blue!6}0.665$_{\pm .012}$} \\
 Qwen-3-30B-A3B   & 0.580$_{\pm .014}$ & 0.645$_{\pm .013}$ & 0.780$_{\pm .011}$ & 0.765$_{\pm .012}$ & 0.785$_{\pm .010}$ & 0.770$_{\pm .011}$ & 0.772$_{\pm .010}$ & \textbf{\cellcolor{blue!6}0.805$_{\pm .009}$} \\
 GPT-5.4-nano     & 0.505$_{\pm .026}$ & 0.550$_{\pm .021}$ & 0.685$_{\pm .018}$ & 0.630$_{\pm .020}$ & 0.650$_{\pm .019}$ & 0.645$_{\pm .018}$ & 0.660$_{\pm .017}$ & \textbf{\cellcolor{blue!6}0.730$_{\pm .015}$} \\
 Claude-Haiku-4.5 & 0.515$_{\pm .023}$ & 0.545$_{\pm .021}$ & 0.710$_{\pm .016}$ & 0.665$_{\pm .017}$ & 0.680$_{\pm .016}$ & 0.635$_{\pm .020}$ & 0.690$_{\pm .015}$ & \textbf{\cellcolor{blue!6}0.775$_{\pm .013}$} \\
\midrule

% ===================== TruthfulQA =====================
\rowcolor{gray!10}
\multicolumn{9}{c}{\textbf{TruthfulQA}} \\
\midrule

 Gemma-3-12B      & 0.705$_{\pm .014}$ & 0.710$_{\pm .013}$ & 0.755$_{\pm .011}$ & 0.760$_{\pm .012}$ & 0.775$_{\pm .010}$ & 0.760$_{\pm .011}$ & 0.770$_{\pm .010}$ & \textbf{\cellcolor{blue!6}0.825$_{\pm .009}$} \\
 Llama-3.1-8B     & 0.645$_{\pm .018}$ & 0.685$_{\pm .017}$ & 0.785$_{\pm .013}$ & 0.745$_{\pm .015}$ & 0.750$_{\pm .014}$ & 0.740$_{\pm .015}$ & 0.765$_{\pm .012}$ & \textbf{\cellcolor{blue!6}0.815$_{\pm .011}$} \\
 Qwen-2.5-14B     & 0.725$_{\pm .013}$ & 0.780$_{\pm .012}$ & 0.875$_{\pm .009}$ & 0.820$_{\pm .011}$ & 0.825$_{\pm .011}$ & 0.855$_{\pm .010}$ & 0.845$_{\pm .009}$ & \textbf{\cellcolor{blue!6}0.890$_{\pm .008}$} \\
 Qwen-3-30B-A3B   & 0.740$_{\pm .012}$ & 0.795$_{\pm .011}$ & 0.845$_{\pm .009}$ & 0.800$_{\pm .010}$ & 0.820$_{\pm .010}$ & 0.780$_{\pm .011}$ & 0.835$_{\pm .009}$ & \textbf{\cellcolor{blue!6}0.905$_{\pm .007}$} \\
 GPT-5.4-nano     & 0.635$_{\pm .021}$ & 0.620$_{\pm .022}$ & 0.845$_{\pm .012}$ & 0.790$_{\pm .015}$ & 0.785$_{\pm .015}$ & 0.830$_{\pm .013}$ & 0.835$_{\pm .012}$ & \textbf{\cellcolor{blue!6}0.915$_{\pm .010}$} \\
 Claude-Haiku-4.5 & 0.615$_{\pm .022}$ & 0.695$_{\pm .018}$ & 0.855$_{\pm .012}$ & 0.775$_{\pm .014}$ & 0.785$_{\pm .013}$ & 0.835$_{\pm .012}$ & 0.845$_{\pm .011}$ & \textbf{\cellcolor{blue!6}0.905$_{\pm .010}$} \\
\midrule

% ===================== GSM8K =====================
\rowcolor{gray!10}
\multicolumn{9}{c}{\textbf{GSM8K}} \\
\midrule

 Gemma-3-12B      & 0.805$_{\pm .010}$ & 0.815$_{\pm .009}$ & 0.900$_{\pm .007}$ & 0.915$_{\pm .007}$ & 0.885$_{\pm .008}$ & 0.890$_{\pm .008}$ & 0.905$_{\pm .006}$ & \textbf{\cellcolor{blue!6}0.955$_{\pm .005}$} \\
 Llama-3.1-8B     & 0.785$_{\pm .012}$ & 0.810$_{\pm .011}$ & 0.835$_{\pm .009}$ & 0.865$_{\pm .008}$ & 0.875$_{\pm .008}$ & 0.855$_{\pm .009}$ & 0.860$_{\pm .007}$ & \textbf{\cellcolor{blue!6}0.895$_{\pm .006}$} \\
 Qwen-2.5-14B     & 0.825$_{\pm .009}$ & 0.820$_{\pm .009}$ & 0.890$_{\pm .007}$ & 0.895$_{\pm .007}$ & 0.930$_{\pm .006}$ & 0.915$_{\pm .007}$ & 0.905$_{\pm .006}$ & \textbf{\cellcolor{blue!6}0.945$_{\pm .005}$} \\
 Qwen-3-30B-A3B   & 0.835$_{\pm .008}$ & 0.845$_{\pm .008}$ & 0.960$_{\pm .004}$ & 0.955$_{\pm .004}$ & 0.945$_{\pm .005}$ & 0.950$_{\pm .005}$ & 0.955$_{\pm .004}$ & \textbf{\cellcolor{blue!6}0.980$_{\pm .003}$} \\
 GPT-5.4-nano     & 0.815$_{\pm .011}$ & 0.835$_{\pm .010}$ & 0.905$_{\pm .007}$ & 0.895$_{\pm .008}$ & 0.885$_{\pm .008}$ & 0.880$_{\pm .008}$ & 0.900$_{\pm .007}$ & \textbf{\cellcolor{blue!6}0.950$_{\pm .005}$} \\
 Claude-Haiku-4.5 & 0.845$_{\pm .009}$ & 0.835$_{\pm .010}$ & 0.910$_{\pm .007}$ & 0.890$_{\pm .008}$ & 0.905$_{\pm .007}$ & 0.895$_{\pm .008}$ & 0.915$_{\pm .006}$ & \textbf{\cellcolor{blue!6}0.955$_{\pm .005}$} \\
\midrule

% ===================== MATH =====================
\rowcolor{gray!10}
\multicolumn{9}{c}{\textbf{MATH-500}} \\
\midrule

 Gemma-3-12B      & 0.185$_{\pm .020}$ & 0.205$_{\pm .019}$ & 0.275$_{\pm .016}$ & 0.285$_{\pm .016}$ & 0.295$_{\pm .015}$ & 0.290$_{\pm .015}$ & 0.300$_{\pm .015}$ & \textbf{\cellcolor{blue!6}0.335$_{\pm .013}$} \\
 Llama-3.1-8B     & 0.135$_{\pm .022}$ & 0.165$_{\pm .020}$ & 0.205$_{\pm .018}$ & 0.195$_{\pm .019}$ & 0.225$_{\pm .017}$ & 0.215$_{\pm .018}$ & 0.220$_{\pm .017}$ & \textbf{\cellcolor{blue!6}0.255$_{\pm .015}$} \\
 Qwen-2.5-14B     & 0.255$_{\pm .018}$ & 0.275$_{\pm .017}$ & 0.385$_{\pm .014}$ & 0.395$_{\pm .014}$ & 0.420$_{\pm .013}$ & 0.410$_{\pm .013}$ & 0.405$_{\pm .014}$ & \textbf{\cellcolor{blue!6}0.485$_{\pm .011}$} \\
 Qwen-3-30B-A3B   & 0.345$_{\pm .015}$ & 0.440$_{\pm .013}$ & 0.540$_{\pm .011}$ & 0.560$_{\pm .011}$ & 0.620$_{\pm .010}$ & 0.610$_{\pm .010}$ & 0.585$_{\pm .012}$ & \textbf{\cellcolor{blue!6}0.665$_{\pm .008}$} \\
 GPT-5.4-nano     & 0.315$_{\pm .017}$ & 0.340$_{\pm .016}$ & 0.455$_{\pm .013}$ & 0.445$_{\pm .013}$ & 0.435$_{\pm .014}$ & 0.440$_{\pm .013}$ & 0.460$_{\pm .013}$ & \textbf{\cellcolor{blue!6}0.525$_{\pm .010}$} \\
 Claude-Haiku-4.5 & 0.285$_{\pm .018}$ & 0.355$_{\pm .016}$ & 0.425$_{\pm .014}$ & 0.410$_{\pm .014}$ & 0.405$_{\pm .014}$ & 0.415$_{\pm .013}$ & 0.420$_{\pm .014}$ & \textbf{\cellcolor{blue!6}0.495$_{\pm .011}$} \\
\bottomrule
\end{tabular}
\end{table*}

\begin{figure*}[t]
    \centering
    \includegraphics[width=0.95\linewidth]{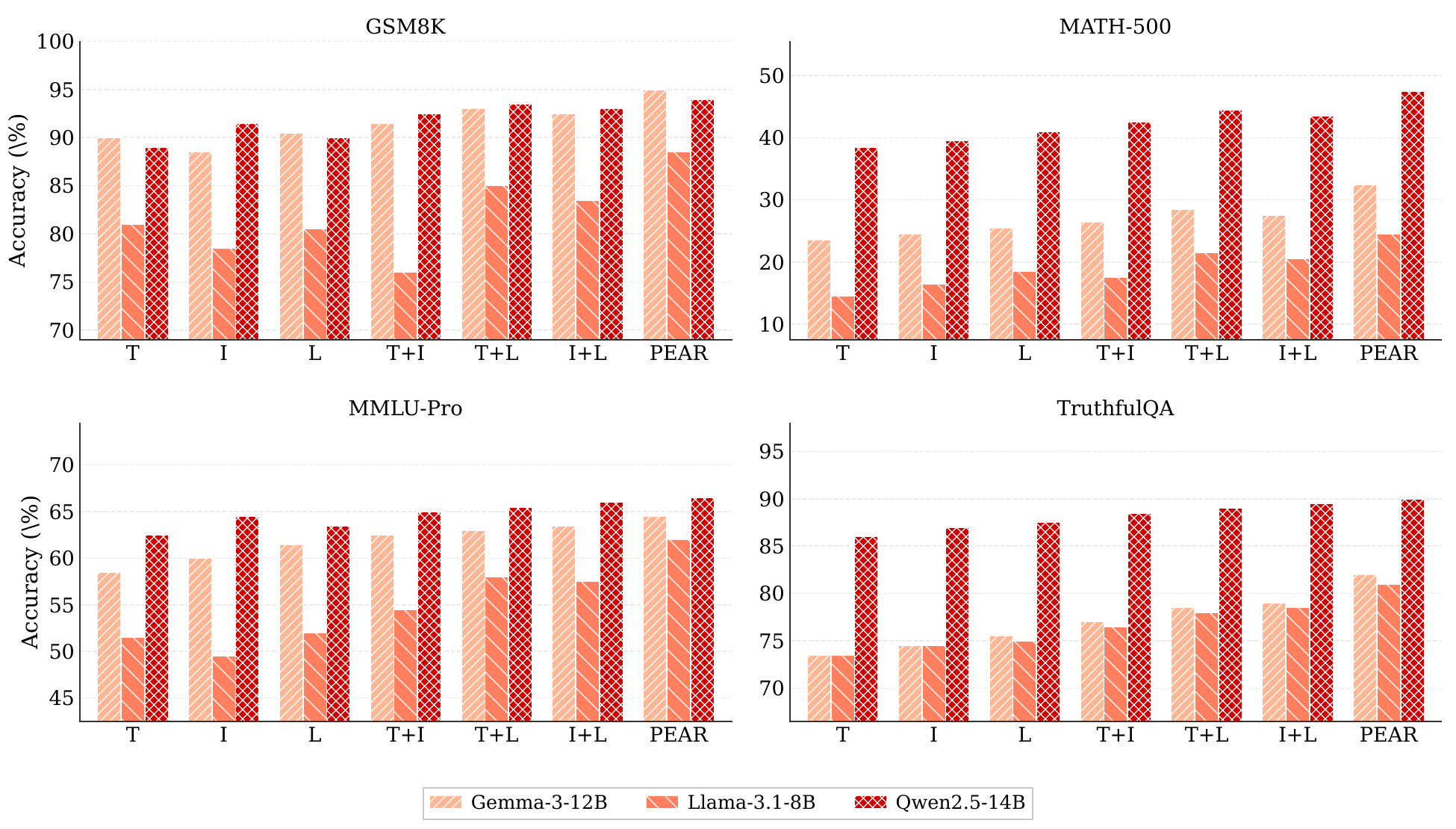}
    \caption{Ablation study of PEAR routing components on open-source models. Targeted = targeted diversity, Influence = influence balancing, LowConf = low-confidence filtering. Each group of bars corresponds to one backbone model; colors indicate routing variant, with PEAR consistently achieving the highest accuracy.}
    \label{fig:ablation}
\end{figure*}

\subsection{Experimental Setup}

\paragraph{Benchmarks.}
We use four benchmarks in knowledge reasoning, factuality, and mathematical reasoning. \textbf{MMLU-Pro}~\cite{wang2024mmlu:mmlu-pro} evaluates broad professional and academic knowledge with multi-step reasoning questions. \textbf{TruthfulQA}~\cite{lin2022truthfulqa:truthfulqa} measures factual reliability and resistance to common misconceptions or false beliefs. \textbf{GSM8K}~\cite{DBLP:journals/corr/abs-2110-14168:gsm8k} focuses on grade-school mathematical word problems that require multi-step arithmetic reasoning. \textbf{MATH-500}~\cite{hendrycks2021measuring:math500} contains competition-level mathematics problems with substantially longer and more complex reasoning chains.

\paragraph{Models.}
We use four open-source instruction-tuned models, namely Gemma-3-12B~\cite{gemmateam2025gemma3technicalreport}, Llama-3.1-8B~\cite{DBLP:journals/corr/abs-2407-21783:llama}, Qwen2.5-14B, and Qwen-3-30B-A3B~\cite{DBLP:journals/corr/abs-2505-09388:qwen3}, together with two closed-source models, GPT-5.4-nano~\cite{openai2026gpt54} and Claude-Haiku-4.5~\cite{anthropic2025haiku45}. All models are used in the main experiments, while ablation studies are conducted on the three smaller open-source models (Gemma-3-12B, Llama-3.1-8B, and Qwen2.5-14B) to reduce computational cost.

% We compare AR-MAD against the following methods: \textbf{CoT}, a single-agent chain-of-thought baseline; \textbf{CoT-SC}, self-consistency over multiple independent single-agent samples; \textbf{Fixed Clique}, a static fully connected debate graph where each agent observes all others; \textbf{Fixed Star}, a static hub-and-spoke debate graph; \textbf{Fixed Chain}, a static sequential communication graph; \textbf{Fixed Ring}, a static cycle graph where agents communicate with fixed local neighbors; \textbf{Random}, a dynamic sparse baseline where each agent receives critiques from $k$ randomly selected peers per round; and \textbf{AR-MAD}, the complete state-aware routing method combining targeted cross-answer exposure, influence balancing, and low-confidence filtering.
\paragraph{Baselines.}
We compare PEAR against the following baselines: \textbf{CoT}~\cite{DBLP:conf/nips/Wei0SBIXCLZ22}, a single-agent chain-of-thought method; \textbf{CoT-SC}~\cite{wang2022self:sc-cot}, self-consistency over multiple independent samples; and a set of fixed-topology multi-agent debate variants including fully connected (\textbf{Clique}), hub-and-spoke (\textbf{Star}), sequential (\textbf{Chain}), and cyclic (\textbf{Ring}) communication graphs. We further include \textbf{Random}, a dynamic sparse baseline where each agent receives critiques from randomly selected peers at each round.

\paragraph{Metrics.}
We report \textit{round-wise} and \textit{final-answer accuracy} as the main results via majority vote, together with five trajectory-level diagnostics: the wrong-to-right and right-to-wrong update rates (\textit{W2R} / \textit{R2W}), the \textit{critique acceptance rate}, the \textit{cross-answer routing rate} (fraction of routed edges whose source and target hold different answers), \textit{source confidence} (mean source-side confidence in $[0,1]$), and \textit{influence entropy} (normalized entropy of the influence distribution). Detailed definitions are given in Appendix~\ref{app:diagnostics}.

\paragraph{Debate configuration.}
PEAR uses $n=5$ agents and $R=5$ debate rounds. The base topology is $k$-regular with $k=2$, so each target agent receives critiques from two source agents per round.
Further implementation details, including hyperparameter selection, are provided in Appendix~\ref{app:debate}.

\paragraph{Reproducibility.} %More implementation details, p
Prompt templates and a case study of multi-round debate dynamics are in Appendices \ref{app:prompts} and \ref{app:case_study}, respectively. 
The code is available at \url{https://github.com/EVIEHub/PEAR}.
%The code is available at \url{https://anonymous.4open.science/r/PEAR-C8AB/README.md}.

\subsection{Experimental Results}

\paragraph{Overall accuracy.}
% Table~\ref{tab:main-results} reports final-answer accuracy across all model--dataset combinations, spanning four reasoning benchmarks, four open-source backbones, and two closed-source backbones. 
% All open-source models are evaluated on 200 examples per dataset, while closed-source models are evaluated on 100 examples due to API cost constraints. Full-scale results on the complete evaluation sets are reported in Appendix~\ref{app:full-results}.

Table~\ref{tab:main-results} reports final-answer accuracy across all model--dataset combinations. Open-source models are evaluated on 200 examples per dataset and closed-source models on 100 due to API cost constraints; full-scale results are reported in Appendix~\ref{app:full-results}. PEAR achieves the highest accuracy in every row, regardless of dataset, model scale, or provenance. Averaged over all settings, PEAR reaches a mean accuracy of $0.701$, compared with $0.620$ for Fixed Clique, $0.602$ for Fixed Star, $0.610$ for Fixed Chain, and $0.609$ for Fixed Ring. Even when each setting is allowed to pick its single most favorable fixed topology, PEAR still gains $5.1$ points on average.

Per-dataset gains over the best fixed topology are $6.0$ points on MMLU-Pro, $3.9$ on TruthfulQA, $4.8$ on GSM8K, and $5.6$ on MATH-500. The improvement is most pronounced on weaker backbones: on Llama-3.1-8B, PEAR gains $8.0$, $9.0$, and $9.5$ points on MMLU-Pro, GSM8K, and MATH-500 respectively, consistent with the failure modes of fixed-topology debate being more severe when individual agents are weaker. Single-agent baselines such as CoT and CoT-SC trail every debate variant by often double-digit margins, and CoT-SC closes only a small portion of the gap, indicating that the gains arise from inter-agent communication rather than from drawing more samples per query.

\paragraph{Ablation study.}
We ablate the three routing components of PEAR: targeted diversity, influence balancing, and low-confidence filtering. Figure~\ref{fig:ablation} reports accuracy for all single-component, two-component, and full variants on 3 open-source models.
The ablation results show that the full routing objective performs best on average. Among partial variants, combining targeted diversity with low-confidence filtering and combining influence balancing with low-confidence filtering are the strongest, reaching average accuracies of $0.603$ and $0.604$, respectively. The full model reaches $0.657$, suggesting that the three components provide complementary benefits. The improvement is particularly large for Llama-3.1-8B, where PEAR Full substantially outperforms all ablations on GSM8K, MATH-500, MMLU-Pro, and TruthfulQA.

\paragraph{Accuracy across debate rounds.}

% To examine how accuracy evolves over the course of a debate, we plot round-wise accuracy for AR-MAD and all baselines on representative open- and closed-source backbones. Figures~\ref{fig:rounds-qwen} and~\ref{fig:rounds-gpt} show the round-by-round trajectories on Qwen-2.5-14B and GPT-5.4-nano, respectively. Across both backbones and all four datasets, AR-MAD consistently achieves higher accuracy at every round and converges to a higher plateau than all fixed-topology baselines. The advantage widens progressively with additional rounds, suggesting that adaptive routing compounds its benefits as the debate proceeds. On Qwen-2.5-14B, AR-MAD improves from $0.284$ at Round~1 to $0.485$ at Round~5 on MATH-500, compared to $0.283$--$0.420$ for the best fixed baseline (Fixed Chain), and from $0.455$ to $0.665$ on MMLU-Pro against $0.453$--$0.620$ for Fixed Chain. On GPT-5.4-nano, the margin is even larger on harder tasks: AR-MAD reaches $0.525$ on MATH-500 by Round~5, surpassing Fixed Clique and Random by $0.070$ points, and achieves $0.730$ on MMLU-Pro against $0.685$ for the next best baseline. On TruthfulQA, where all methods start from a higher base, AR-MAD still maintains a consistent lead, reaching $0.915$ on GPT-5.4-nano compared to $0.845$ for Fixed Clique at Round~5.
%To examine how accuracy evolves over the course of a debate, 
Figures~\ref{fig:rounds-qwen} and~\ref{fig:rounds-gpt} in Appendix~\ref{app:rounds} plot round-wise accuracy on Qwen-2.5-14B and GPT-5.4-nano respectively. Across both backbones and all four datasets, PEAR achieves higher accuracy at every round and converges to a higher plateau than every fixed-topology baseline, with the advantage widening progressively as the debate proceeds. For example, on MATH-500 with Qwen-2.5-14B, PEAR improves from $0.284$ at Round~1 to $0.485$ at Round~5, whereas the strongest fixed baseline saturates at $0.420$. The trend is sharper on harder tasks with GPT-5.4-nano: PEAR reaches $0.525$ on MATH-500 and $0.730$ on MMLU-Pro by Round~5, exceeding the next-best baseline by $0.07$ and $0.045$ points respectively. Even on TruthfulQA, where all methods start from a higher base, PEAR maintains a consistent lead, reaching $0.915$ at Round~5 against $0.845$ for Fixed Clique.

\paragraph{Trajectory-level diagnostics.}
% The diagnostics in Table~\ref{tab:diagnostics} show that AR-MAD improves both correction quality and routing diversity compared to all baselines. It achieves the highest net correction rate of $0.243$, indicating that beneficial updates substantially outweigh harmful flips. More importantly, AR-MAD exhibits substantially higher cross-answer routing at $0.676$, compared to $0.43$--$0.50$ for all baselines, suggesting that the adaptive mechanism effectively encourages information exchange across diverse reasoning trajectories rather than reinforcing local agreement. AR-MAD also yields the highest influence entropy at $0.979$, reflecting a more balanced distribution of influence across agents. Together, these results indicate that performance gains stem from both improved correction dynamics and more diverse communication patterns.
% Detailed metric definitions and additional pattern analyses are provided in Appendix~\ref{app:diagnostics}. 
The diagnostics in Table~\ref{tab:diagnostics} in Appendix~\ref{app:diagnostics} show that PEAR improves both correction quality and routing diversity compared to all baselines. It achieves the highest net correction rate of $0.243$, indicating that beneficial updates substantially outweigh harmful flips, and exhibits substantially higher cross-answer routing at $0.676$ versus $0.43$--$0.50$ for all baselines, suggesting that adaptive routing encourages information exchange across diverse reasoning trajectories rather than reinforcing local agreement. PEAR also yields the highest influence entropy at $0.979$, reflecting a more balanced distribution of influence across agents. Detailed metric definitions and analyses are provided in Appendix~\ref{app:diagnostics}.

\paragraph{Computational overhead.}
% As shown in Figure~\ref{fig:pareto}, AR-MAD lies on the cost--accuracy Pareto frontier on every dataset. We further measure per-example token usage to assess its computational overhead relative to existing debate baselines. AR-MAD consumes about $32.3$k tokens per example, comparable to the Random baseline, which uses the same $k$-regular base graph but selects per-round routing uniformly at random, and $12\%$ lower than the densest Fixed Clique baseline at $36.7$k tokens. Adaptive routing therefore introduces \emph{negligible} additional cost over uniformly random routing under a similar edge budget, while improving average accuracy by $5.7$ points over Random and $6.6$ points over the cheaper Fixed Chain baseline. More details are provided in Appendix~\ref{app:overhead}.
As shown in Figure~\ref{fig:pareto} in Appendix~\ref{app:overhead}, PEAR lies on the cost--accuracy Pareto frontier on every dataset. It consumes about $32.3$k tokens per example, comparable to the Random baseline at the same $k$-regular edge budget and $12\%$ lower than the densest Fixed Clique baseline at $36.7$k tokens. Adaptive routing therefore adds \emph{negligible} cost over uniformly random routing under a similar edge budget, while improving accuracy by $5.7$ points over Random and $6.6$ points over the cheaper Fixed Chain baseline. More details are provided in Appendix~\ref{app:overhead}.

\section{Conclusions}
This paper presents \textit{Permutation-Equivariant Adaptive Routing Multi-Agent Debate (PEAR)}, an inference-time framework that selects a sparse communication topology at each debate round based on agent states. This state-aware routing actively suppresses information homogenization and error cascades, the two common failure modes of fixed-topology debate. Extensive experiments on four benchmarks and six LLMs show that PEAR consistently outperforms fixed-topology and dynamic baselines, with gains driven by improved correction dynamics and more diverse communication.

\section*{Limitations}

\paragraph{Computational overhead.}
Although PEAR lies on the cost--accuracy Pareto frontier among debate methods (see in Appendix~\ref{app:overhead}), its per-example token usage %($\sim\!32$k) 
remains higher than %roughly $40\times$ that of 
single-agent CoT %($\sim\!0.8$k) and $10\times$ that of 
and CoT-SC. %($\sim\!3.1$k). For latency-sensitive or cost-constrained deployments this overhead is the primary practical consideration when adopting multi-agent debate. 
Promising directions for reducing cost while preserving routing quality include adaptive termination, i.e., stopping rounds once a confidence-weighted consensus is reached, and sparser candidate pools; we leave a systematic study of these complementary techniques to future work.

\paragraph{Reliance on self-reported confidence.}
Two of PEAR's three routing components, Targeted Diversity and Low-Confidence Filtering, directly consume the self-reported confidence scores produced by the agents themselves. While LLM confidence is known to be imperfectly calibrated in some settings, our results suggest the router is reasonably robust to this in practice: routing decisions are based on relative confidence comparisons across agents rather than absolute thresholds, which partially mitigates the effect of systematic bias. Nonetheless, coupling PEAR with an external confidence estimator or a calibration-aware routing objective represents a promising direction for further improving routing precision.
 
%\paragraph{Hyperparameter sensitivity.}
%The routing-objective weights $(\alpha_T,\alpha_I,\alpha_L)$, confidence thresholds $\tau_{\mathrm{src}}$, $\tau_{\mathrm{tgt}}$, $\tau_{\mathrm{low}}$, and influence-smoothing coefficient $\beta$ are held fixed across every dataset and backbone in our experiments. We do not perform a large-scale sensitivity analysis and do not provide an automatic tuning procedure. Under substantial task distribution shifts or markedly different model populations, some of these values may need to be re-selected, and how to do so robustly without held-out validation data is left to future work.

\section*{Ethics Considerations}
\label{sec:ethics}

%This work is fundamental research on inference-time communication for multi-agent reasoning with large language models. 
All experiments use publicly available academic benchmarks and involve no human subjects or sensitive data. We do not foresee any direct negative societal impacts arising from the proposed framework.

%\section*{Acknowledgements}
%We thank our collaborators and colleagues for their valuable discussions, insightful suggestions, and constructive feedback, which helped shape the ideas and improve this work. We are also grateful for their helpful perspectives and encouragement throughout the project, as well as for providing the computational resources and technical support needed to conduct our experiments.

%\begin{ack}
%Use unnumbered first level headings for the acknowledgments. All acknowledgments
%go at the end of the paper before the list of references. Moreover, you are required to declare
%funding (financial activities supporting the submitted work) and competing interests (related financial activities outside the submitted work).
%More information about this disclosure can be found at: \url{https://neurips.cc/Conferences/2025/PaperInformation/FundingDisclosure}.

%Do {\bf not} include this section in the anonymized submission, only in the final paper. You can use the \texttt{ack} environment provided in the style file to automatically hide this section in the anonymized submission.
%\end{ack}

%\bibliographystyle{plain}
%\newpage
\bibliography{References}

%%%%%%%%%%%%%%%%%%%%%%%%%%%%%%%%%%%%%%%%%%%%%%%%%%%%%%%%%%%%

\appendix

\newpage
\onecolumn

\section{Proofs}
\label{app:theory-proofs}

\subsection{Proof of Theorem~\ref{thm:armad-equivariant}}
\label{app:armad-equivariant}

\begin{proof}
    
For a debate state $s$, let $y_i(s)$, $c_i(s)$, and $\rho_i(s)$ denote the
answer, confidence, and accumulated influence of agent $i$. Under relabeling by
$\sigma$,
\[
    y_{\sigma i}(\sigma s)=y_i(s),
    \qquad
    c_{\sigma i}(\sigma s)=c_i(s),
    \qquad
    \rho_{\sigma i}(\sigma s)=\rho_i(s).
\]
For any edge $(u,v)\in E$, the targeted-diversity indicator satisfies
\[
\begin{aligned}
    T(\sigma u,\sigma v\mid \sigma s)
    &=
    \mathbf 1\{y_{\sigma u}(\sigma s)\ne y_{\sigma v}(\sigma s)\}
    \mathbf 1\{c_{\sigma u}(\sigma s)\ge \tau_{\rm src}\}
    \mathbf 1\{c_{\sigma v}(\sigma s)\le \tau_{\rm tgt}\} \\
    &=
    \mathbf 1\{y_u(s)\ne y_v(s)\}
    \mathbf 1\{c_u(s)\ge \tau_{\rm src}\}
    \mathbf 1\{c_v(s)\le \tau_{\rm tgt}\} \\
    &=T(u,v\mid s).
\end{aligned}
\]
Since $\sigma:E\to\sigma E$ is a bijection and all candidates have $m$ edges,
\[
    \widetilde T(\sigma E\mid \sigma s)
    =
    \frac{1}{m}
    \sum_{(\sigma u,\sigma v)\in\sigma E}
    T(\sigma u,\sigma v\mid \sigma s)
    =
    \widetilde T(E\mid s).
\]

For the influence term, graph relabeling gives
$d_{\sigma u}^{\rm out}(\sigma E)=d_u^{\rm out}(E)$. Hence
\[
\begin{aligned}
    \widetilde I(\sigma E\mid \sigma s)
    &=
    \frac{1}{m}
    \sum_{i=1}^n
    \rho_i(\sigma s)d_i^{\rm out}(\sigma E) \\
    &=
    \frac{1}{m}
    \sum_{u=1}^n
    \rho_{\sigma u}(\sigma s)d_{\sigma u}^{\rm out}(\sigma E) \\
    &=
    \frac{1}{m}
    \sum_{u=1}^n
    \rho_u(s)d_u^{\rm out}(E)
    =
    \widetilde I(E\mid s).
\end{aligned}
\]
Similarly, because $L(\sigma u\mid \sigma s)=L(u\mid s)$,
\[
    \widetilde L(\sigma E\mid \sigma s)
    =
    \widetilde L(E\mid s).
\]
Combining the three identities yields
\[
    S_\alpha(\sigma E\mid \sigma s)=S_\alpha(E\mid s).
\]

It remains to verify the softmax distribution. By closure of the candidate
family, the map $E'\mapsto \sigma E'$ is a bijection from $\mathcal C(s)$ to
$\mathcal C(\sigma s)$. Therefore
\[
\begin{aligned}
    Q_\alpha(\sigma E\mid \sigma s)
    &=
    \frac{
    \exp(S_\alpha(\sigma E\mid \sigma s)/\tau)
    }{
    \sum_{F\in\mathcal C(\sigma s)}
    \exp(S_\alpha(F\mid \sigma s)/\tau)
    }\\
    &=
    \frac{
    \exp(S_\alpha(E\mid s)/\tau)
    }{
    \sum_{E'\in\mathcal C(s)}
    \exp(S_\alpha(\sigma E'\mid \sigma s)/\tau)
    }\\
    &=
    \frac{
    \exp(S_\alpha(E\mid s)/\tau)
    }{
    \sum_{E'\in\mathcal C(s)}
    \exp(S_\alpha(E'\mid s)/\tau)
    }
    =
    Q_\alpha(E\mid s).
\end{aligned}
\]
For $\tau\to0$, the softmax concentrates on the maximizers of
$S_\alpha(\cdot\mid s)$. If ties are broken uniformly among maximizers, the
same bijection argument applies to the argmax set, giving equivariant
deterministic routing. 
\end{proof}

\subsection{Proof of Theorem~\ref{thm:accuracy-preservation}}
\label{app:accuracy-preservation}

\begin{proof}
    
First, we show that $\mathcal Pq$ is equivariant. For any
$\sigma_0\in\mathfrak S_n$,
\[
\begin{aligned}
    (\mathcal Pq)(\sigma_0\Gamma\mid \sigma_0s^{(0)})
    &=
    \frac{1}{n!}
    \sum_{\sigma\in\mathfrak S_n}
    q(\sigma\sigma_0\Gamma\mid \sigma\sigma_0s^{(0)}).
\end{aligned}
\]
As $\sigma$ ranges over $\mathfrak S_n$, so does $\sigma\sigma_0$. Thus
\[
    (\mathcal Pq)(\sigma_0\Gamma\mid \sigma_0s^{(0)})
    =
    \frac{1}{n!}
    \sum_{\gamma\in\mathfrak S_n}
    q(\gamma\Gamma\mid \gamma s^{(0)})
    =
    (\mathcal Pq)(\Gamma\mid s^{(0)}).
\]

Now consider expected accuracy:
\[
\begin{aligned}
    \mathbb E_{\mathcal Pq}[A]
    &=
    \mathbb E_{s^{(0)}}
    \sum_\Gamma
    (\mathcal Pq)(\Gamma\mid s^{(0)})A(\Gamma,s^{(0)})\\
    &=
    \frac{1}{n!}
    \sum_{\sigma\in\mathfrak S_n}
    \mathbb E_{s^{(0)}}
    \sum_\Gamma
    q(\sigma\Gamma\mid \sigma s^{(0)})A(\Gamma,s^{(0)}).
\end{aligned}
\]
Let $\Gamma'=\sigma\Gamma$ and $s'=\sigma s^{(0)}$. Since the distribution of
$s^{(0)}$ is exchangeable, $s'$ has the same distribution as $s^{(0)}$.
Permutation-equivariance of the update kernels and permutation-invariance of
the aggregation rule imply
\[
    A(\Gamma,s^{(0)})
    =
    A(\sigma\Gamma,\sigma s^{(0)})
    =
    A(\Gamma',s').
\]
Therefore every summand equals
\[
    \mathbb E_{s'}
    \sum_{\Gamma'}
    q(\Gamma'\mid s')A(\Gamma',s'),
\]
which is exactly the expected accuracy under $q$. Averaging over $\sigma$ leaves
the same quantity. 
\end{proof}

\subsection{Proofs of Lemma~\ref{lem:covering-projection} and Theorem~\ref{thm:generalization}}
\label{app:complexity}

\begin{definition}[$\epsilon$-covering number]
\label{def:covering-number}
Let $(\mathcal F,d)$ be a metric space. A set
$\mathcal C_\epsilon\subseteq \mathcal F$ is an $\epsilon$-cover of
$\mathcal F$ if, for every $f\in\mathcal F$, there exists
$g\in\mathcal C_\epsilon$ such that $d(f,g)\le \epsilon$. The covering number is
\[
    \mathcal N(\mathcal F,\epsilon,d)
    =
    \min\left\{
    |\mathcal C_\epsilon|:
    \mathcal C_\epsilon
    \text{ is an } \epsilon\text{-cover of } \mathcal F
    \right\}.
\]
\end{definition}

\begin{proof}[Proof of Lemma~\ref{lem:covering-projection}]

We first show that $\mathcal P$ is non-expansive under
\[
    d(q,q')=\sup_s\|q(\cdot\mid s)-q'(\cdot\mid s)\|_1 .
\]
For any state $s$,
\[
\begin{aligned}
    &
    \|(\mathcal Pq)(\cdot\mid s)-(\mathcal Pq')(\cdot\mid s)\|_1 \\
    &=
    \left\|
    \frac{1}{n!}
    \sum_{\sigma\in\mathfrak S_n}
    \left[
    q(\sigma\cdot\mid \sigma s)
    -
    q'(\sigma\cdot\mid \sigma s)
    \right]
    \right\|_1 \\
    &\le
    \frac{1}{n!}
    \sum_{\sigma\in\mathfrak S_n}
    \left\|
    q(\sigma\cdot\mid \sigma s)
    -
    q'(\sigma\cdot\mid \sigma s)
    \right\|_1 \\
    &\le d(q,q').
\end{aligned}
\]
Taking the supremum over $s$ gives
\[
    d(\mathcal Pq,\mathcal Pq')\le d(q,q').
\]

Let $\{q_1,\ldots,q_M\}$ be an $\epsilon$-cover of $\mathcal Q$. For any
$\mathcal Pq\in\mathcal P\mathcal Q$, there exists $q_j$ such that
$d(q,q_j)\le\epsilon$. By non-expansiveness,
\[
    d(\mathcal Pq,\mathcal Pq_j)\le \epsilon.
\]
Thus
\[
    \{\mathcal Pq_1,\ldots,\mathcal Pq_M\}
\]
is an $\epsilon$-cover of $\mathcal P\mathcal Q$. Taking the smallest possible
$M$ proves
\[
    \mathcal N(\mathcal P\mathcal Q,\epsilon,d)
    \le
    \mathcal N(\mathcal Q,\epsilon,d).
\]
\end{proof}

%\subsection{Proof of Theorem~\ref{thm:generalization}}
%\label{app:generalization}

\begin{proof}[Proof of Theorem~\ref{thm:generalization}]
    
Let
\[
    M=\mathcal N(\mathcal P\mathcal Q,\epsilon,d),
\]
and let
\[
    \{q_1,\ldots,q_M\}
\]
be an $\epsilon$-cover of $\mathcal P\mathcal Q$. For any fixed $q_j$,
Hoeffding's inequality gives
\[
    \Pr\left(
    \left|
    \mathcal L(q_j)-\widehat{\mathcal L}_{N_{\rm ex}}(q_j)
    \right|
    >
    t
    \right)
    \le
    2\exp(-2N_{\rm ex}t^2),
\]
because $F(q_j;z)\in[0,1]$. A union bound over the cover implies that, with
probability at least $1-\delta$,
\[
    \max_{1\le j\le M}
    \left|
    \mathcal L(q_j)-\widehat{\mathcal L}_{N_{\rm ex}}(q_j)
    \right|
    \le
    \sqrt{
    \frac{
    \log(2M/\delta)
    }{
    2N_{\rm ex}
    }}
    .
\]

Now fix any $q\in\mathcal P\mathcal Q$, and choose $q_j$ from the cover such
that $d(q,q_j)\le\epsilon$. By the $L$-Lipschitz condition,
\[
    |\mathcal L(q)-\mathcal L(q_j)|\le L\epsilon,
    \qquad
    |\widehat{\mathcal L}_{N_{\rm ex}}(q)
    -
    \widehat{\mathcal L}_{N_{\rm ex}}(q_j)|
    \le L\epsilon.
\]
Therefore
\[
\begin{aligned}
    \left|
    \mathcal L(q)-\widehat{\mathcal L}_{N_{\rm ex}}(q)
    \right|
    &\le
    |\mathcal L(q)-\mathcal L(q_j)|
    +
    |\mathcal L(q_j)-\widehat{\mathcal L}_{N_{\rm ex}}(q_j)|\\
    &\qquad
    +
    |\widehat{\mathcal L}_{N_{\rm ex}}(q_j)
    -
    \widehat{\mathcal L}_{N_{\rm ex}}(q)|\\
    &\le
    2L\epsilon
    +
    \sqrt{
    \frac{
    \log(2M/\delta)
    }{
    2N_{\rm ex}
    }}
    .
\end{aligned}
\]
Taking the supremum over $q\in\mathcal P\mathcal Q$ proves the first bound.
The second follows immediately from Lemma~\ref{lem:covering-projection}.
\end{proof}

\section{One-Round Correction and Influence-Balancing}
\label{app:deferred-theory}

This appendix contains the one-round correction and influence-balancing results that support the interpretation of the PEAR score. %They are omitted from the main text to keep the theoretical section focused on equivariance and generalization.

\subsection{One-round correction lower bound}

Let
\[
    Z_i^{(r)}=\mathbf 1\{y_i^{(r)}=y^\star\}
\]
denote whether agent $i$ is correct after round $r$. Define the expected
one-round improvement under candidate edge set $E$ and state $s$ by
\[
    \Delta(E;s)
    =
    \mathbb E\left[
    \sum_{i=1}^n Z_i^{(r)}
    -
    \sum_{i=1}^n Z_i^{(r-1)}
    \,\middle|\, E,s
    \right].
\]

\begin{assumption}[Additive edge-wise improvement model]
\label{assump:additive-edge-improvement}
For each candidate edge set $E$, the expected improvement at each target is the
sum of marginal contributions from its incoming edges:
\[
    \mathbb E\left[
    Z_v^{(r)}-Z_v^{(r-1)}
    \,\middle|\, E,s
    \right]
    =
    \sum_{u:(u,v)\in E}\delta_{uv}(s).
\]
Moreover, there exist constants
$\lambda_T,\lambda_I,\lambda_L\ge0$ and $\varepsilon\ge0$ such that every
marginal contribution satisfies
\[
    \delta_{uv}(s)
    \ge
    \lambda_T T(u,v\mid s)
    -
    \lambda_I\rho_u(s)
    -
    \lambda_L\frac{L(u\mid s)}{\tau_{\rm low}}
    -
    \varepsilon .
\]
\end{assumption}

The assumption encodes three effects: confident disagreement toward uncertain
targets is helpful, low-confidence sources are unreliable, and repeatedly
amplified sources have diminishing marginal value.

\begin{theorem}[One-round correction lower bound]
\label{thm:surrogate-appendix}
Under Assumption~\ref{assump:additive-edge-improvement}, for any candidate edge
set $E$,
\[
    \Delta(E;s)
    \ge
    m\left(
    \lambda_T\widetilde T(E\mid s)
    -
    \lambda_I\widetilde I(E\mid s)
    -
    \lambda_L\widetilde L(E\mid s)
    -
    \varepsilon
    \right).
\]
If PEAR uses weights
$\alpha=(\lambda_T,\lambda_I,\lambda_L)$ and samples
$E\sim Q_\alpha(\cdot\mid s)$, then
\[
    \mathbb E_{E\sim Q_\alpha}[\Delta(E;s)]
    \ge
    m\left[
    \max_{E'\in\mathcal C(s)}S_\alpha(E'\mid s)
    -
    \tau\log|\mathcal C(s)|
    -
    \varepsilon
    \right].
\]
Thus, as $\tau\to0$, PEAR greedily maximizes a lower bound on expected
one-round correction.
\end{theorem}

\begin{proof}
Assumption~\ref{assump:additive-edge-improvement} models the improvement at
each target $v$ as a sum of marginal contributions over its incoming edges.
Therefore,
\[
    \Delta(E;s)
    =
    \sum_{(u,v)\in E}\delta_{uv}(s).
\]
Using the marginal lower bound and summing over edges gives
\[
\begin{aligned}
    \Delta(E;s)
    &\ge
    \sum_{(u,v)\in E}
    \left[
    \lambda_T T(u,v\mid s)
    -
    \lambda_I\rho_u(s)
    -
    \lambda_L\frac{L(u\mid s)}{\tau_{\rm low}}
    -
    \varepsilon
    \right]\\
    &=
    m\lambda_T\widetilde T(E\mid s)
    -
    m\lambda_I\widetilde I(E\mid s)
    -
    m\lambda_L\widetilde L(E\mid s)
    -
    m\varepsilon\\
    &=
    m\left(
    \lambda_T\widetilde T(E\mid s)
    -
    \lambda_I\widetilde I(E\mid s)
    -
    \lambda_L\widetilde L(E\mid s)
    -
    \varepsilon
    \right).
\end{aligned}
\]

For the softmax guarantee, write
\[
    a_E=S_\alpha(E\mid s),
    \qquad
    p_E=Q_\alpha(E\mid s).
\]
The Gibbs variational identity gives
\[
    \tau\log\sum_{E\in\mathcal C(s)}\exp(a_E/\tau)
    =
    \sum_E p_Ea_E+\tau H(p),
\]
where $H(p)=-\sum_Ep_E\log p_E$. Since
\[
    \tau\log\sum_E\exp(a_E/\tau)\ge \max_E a_E
    \qquad\text{and}\qquad
    H(p)\le \log|\mathcal C(s)|,
\]
we obtain
\[
    \sum_E p_Ea_E
    \ge
    \max_E a_E-\tau\log|\mathcal C(s)|.
\]
Taking expectation in the first part of the theorem with
$E\sim Q_\alpha(\cdot\mid s)$ proves
\[
    \mathbb E_{E\sim Q_\alpha}[\Delta(E;s)]
    \ge
    m\left[
    \max_{E'\in\mathcal C(s)}S_\alpha(E'\mid s)
    -
    \tau\log|\mathcal C(s)|
    -
    \varepsilon
    \right].
\]
\end{proof}

\subsection{Influence monotonicity}

Let
\[
    D_j(E)=d_j^{\rm out}(E)
\]
be the number of targets receiving critiques from source agent $j$.

\begin{theorem}[Influence penalty reduces future amplification]
\label{thm:influence-appendix}
For softmax routing with $\tau>0$,
\[
    \frac{\partial}{\partial \rho_j}
    \mathbb E_{E\sim Q_\alpha(\cdot\mid s)}[D_j(E)]
    =
    -
    \frac{\alpha_I}{\tau m}
    \operatorname{Var}_{E\sim Q_\alpha(\cdot\mid s)}[D_j(E)]
    \le 0 .
\]
Thus, increasing an agent's accumulated influence can only decrease its expected
future out-degree.
\end{theorem}

\begin{proof}
Fix a state $s$ and write the score as
\[
    S_\alpha(E\mid s)
    =
    B(E\mid s)
    -
    \frac{\alpha_I}{m}
    \sum_{i=1}^n \rho_iD_i(E),
\]
where $B(E\mid s)$ contains the targeted-diversity and low-confidence terms,
which do not depend on $\rho_j$. The softmax distribution is
\[
    Q_\alpha(E\mid s)
    =
    \frac{\exp(S_\alpha(E\mid s)/\tau)}{Z(s)}.
\]
By the score-function, or log-derivative, identity, for any statistic $f(E)$,
\[
    \frac{\partial}{\partial \rho_j}
    \mathbb E_{Q_\alpha}[f(E)]
    =
    \operatorname{Cov}_{Q_\alpha}
    \left(
    f(E),
    \frac{\partial}{\partial\rho_j}
    \log Q_\alpha(E\mid s)
    \right).
\]
Moreover,
\[
    \frac{\partial}{\partial\rho_j}
    \log Q_\alpha(E\mid s)
    =
    -
    \frac{\alpha_I}{\tau m}D_j(E)
    -
    \frac{\partial}{\partial\rho_j}\log Z(s).
\]
The second term is constant in $E$ and therefore vanishes inside the covariance.
Taking $f(E)=D_j(E)$ yields
\[
\begin{aligned}
    \frac{\partial}{\partial \rho_j}
    \mathbb E_{Q_\alpha}[D_j(E)]
    &=
    -
    \frac{\alpha_I}{\tau m}
    \operatorname{Cov}_{Q_\alpha}(D_j(E),D_j(E))\\
    &=
    -
    \frac{\alpha_I}{\tau m}
    \operatorname{Var}_{Q_\alpha}(D_j(E)).
\end{aligned}
\]
Since $\alpha_I\ge0$, $\tau>0$, $m>0$, and variance is nonnegative, the
derivative is nonpositive.
\end{proof}
\newpage
\twocolumn

\section{Additional Experimental Details}
\label{app:exps}

This appendix presents additional experimental details.

\subsection{Debate Configuration}
\label{app:debate}

\paragraph{Debate protocol.}
Our experiments use $n=5$ agents and $R=5$ debate rounds. Each round consists of a critique phase followed by an answer-revision phase. Final predictions are aggregated using majority voting. For fixed-topology baselines, we evaluate clique, ring, star, and chain communication structures. For PEAR, the base interaction graph is a $k$-regular graph with degree $k=2$, and the topology is adaptively reconfigured at each round through role-to-agent routing permutations. The maximum generation budget is 512 tokens per model call.

\paragraph{Routing-objective weights.}
The composite routing score is
\begin{equation}
\begin{aligned}
S(E \mid s_r)
&= \alpha_T \widetilde{T}(E \mid s_r)
 - \alpha_I \widetilde{I}(E) \\
&\quad
 - \alpha_L \widetilde{L}(E \mid s_r),
\end{aligned}
\end{equation}
where $\widetilde{T}$, $\widetilde{I}$, and $\widetilde{L}$ are the normalized routing components defined in
Equations~\ref{eq:targeted-diversity-rate}, ~\ref{eq:influence-balancing}, and ~\ref{eq:lcf-rate}, respectively. 
These quantities are normalized through bounded averaging over the candidate edge set and therefore lie in $[0,1]$.
Unless otherwise specified, we set $(\alpha_T,\alpha_I,\alpha_L)=(0.4,\,0.7,\,0.7)$. The asymmetric weighting reflects a deliberate priority among the three components: we first ensure that routing is balanced in influence ($\alpha_I=0.7$) and gated by source confidence ($\alpha_L=0.7$), since these two terms directly address the structural failure modes of fixed-topology debate, namely positional dominance and unreliable senders. Targeted diversity ($\alpha_T=0.4$) then acts as a refinement layered on top of this stabilized backbone, encouraging cross-answer routing once influence and confidence are already controlled, rather than as a primary driver that could over-amplify divergent but noisy critiques. These weights are held fixed across rounds and across datasets.
 
\paragraph{Confidence and influence smoothing.}
Self-reported confidence is solicited on the ordinal scale $c_i^{(r)}\in\{1,\ldots,5\}$. The accumulated-influence statistic is updated with smoothing coefficient $\beta=0.5$ via
\begin{equation*}
    \rho_s^{(r)}=\beta\rho_s^{(r-1)}+(1-\beta)\,a_s^{(r)},\quad \rho_s^{(0)}=0,
\end{equation*}
where $a_s^{(r)}\in[0,1]$ is the per-round adoption rate computed from the ACCEPT/REJECT decisions in agents' revisions. Confidence and source thresholds are $\tau_{\mathrm{src}}=4$, $\tau_{\mathrm{tgt}}=3$, and $\tau_{\mathrm{low}}=2$ on the $1$--$5$ scale.
 
\paragraph{Candidate pool and routing selection.}
At each round, we construct a candidate pool
$\mathcal{C}_r\subseteq S_n$
by uniformly sampling role-to-agent permutations without replacement subject to the base graph structure.
The candidate pool defines the finite search space over which routing scores are evaluated.

For each candidate assignment $\pi\in\mathcal{C}_r$, we compute the composite routing score
$S(E(\pi)\mid s_r)$.
The router then samples the routing assignment according to the state-conditional softmax distribution
\begin{equation*}
Q_r(\pi\mid s_r)
=
\frac{
\exp\!\big(S(E(\pi)\mid s_r)/\tau\big)
}{
\sum_{\pi'\in\mathcal{C}_r}
\exp\!\big(S(E(\pi')\mid s_r)/\tau\big)
},
\end{equation*}
where $\tau$ is the routing temperature.
As $\tau\to 0$, the routing policy approaches greedy argmax selection.

\paragraph{Compute infrastructure.}
All experiments are conducted on a single node equipped with 4 NVIDIA H100 80GB GPUs. Each experiment is run with data-parallel or sequential execution depending on the setting. This setup ensures reproducibility under a fixed compute budget.

\subsection{Full-Dataset Results}
\label{app:full-results}

Table~\ref{tab:full-results} reports results on the full evaluation sets for two representative open-source models. Compared with the results in Section~\ref{sec:experiments}, which uses sampled subsets for computational efficiency, these experiments evaluate all methods on the complete TruthfulQA and MATH-500 benchmarks.

Across both datasets and models, PEAR consistently achieves the best final-answer accuracy. On TruthfulQA, the gains over fixed-topology baselines remain substantial, indicating that adaptive communication continues to improve factual robustness even at full evaluation scale. On MATH-500, where multi-step reasoning errors frequently propagate through the debate process, PEAR also maintains a clear advantage over static topologies, particularly compared with star and ring structures.

These results demonstrate that the improvements observed in the main experiments are not artifacts of subset evaluation and remain stable when scaling to the full datasets.

\begin{table}[h]
\centering
\scriptsize
\setlength{\tabcolsep}{3pt}
\renewcommand{\arraystretch}{1.05}
\caption{
Final-answer accuracy on full datasets.
%Best results are bolded.
}
\label{tab:full-results}
\begin{tabular}{lcccc}
\toprule
\textbf{Method}
& \multicolumn{2}{c}{\textbf{TruthfulQA}}
& \multicolumn{2}{c}{\textbf{MATH-500}} \\
\cmidrule(lr){2-3}
\cmidrule(lr){4-5}
& Qwen2.5-14B & Llama-3.1-8B
& Qwen2.5-14B & Llama-3.1-8B \\
\midrule
CoT
& 0.709 & 0.653
& 0.212 & 0.138 \\

CoT-SC
& 0.785 & 0.675
& 0.256 & 0.142 \\

Clique
& 0.857 & 0.785
& 0.420 & 0.206 \\

Star
& 0.834 & 0.722
& 0.334 & 0.208 \\

Chain
& 0.795 & 0.745
& 0.348 & 0.290 \\

Ring
& 0.827 & 0.757
& 0.344 & 0.218 \\

PEAR
& \textbf{0.868} & \textbf{0.824}
& \textbf{0.442} & \textbf{0.298} \\
\bottomrule
\end{tabular}
\end{table}

%\newpage

\begin{figure*}[t]
    \centering
    \includegraphics[width=\linewidth]{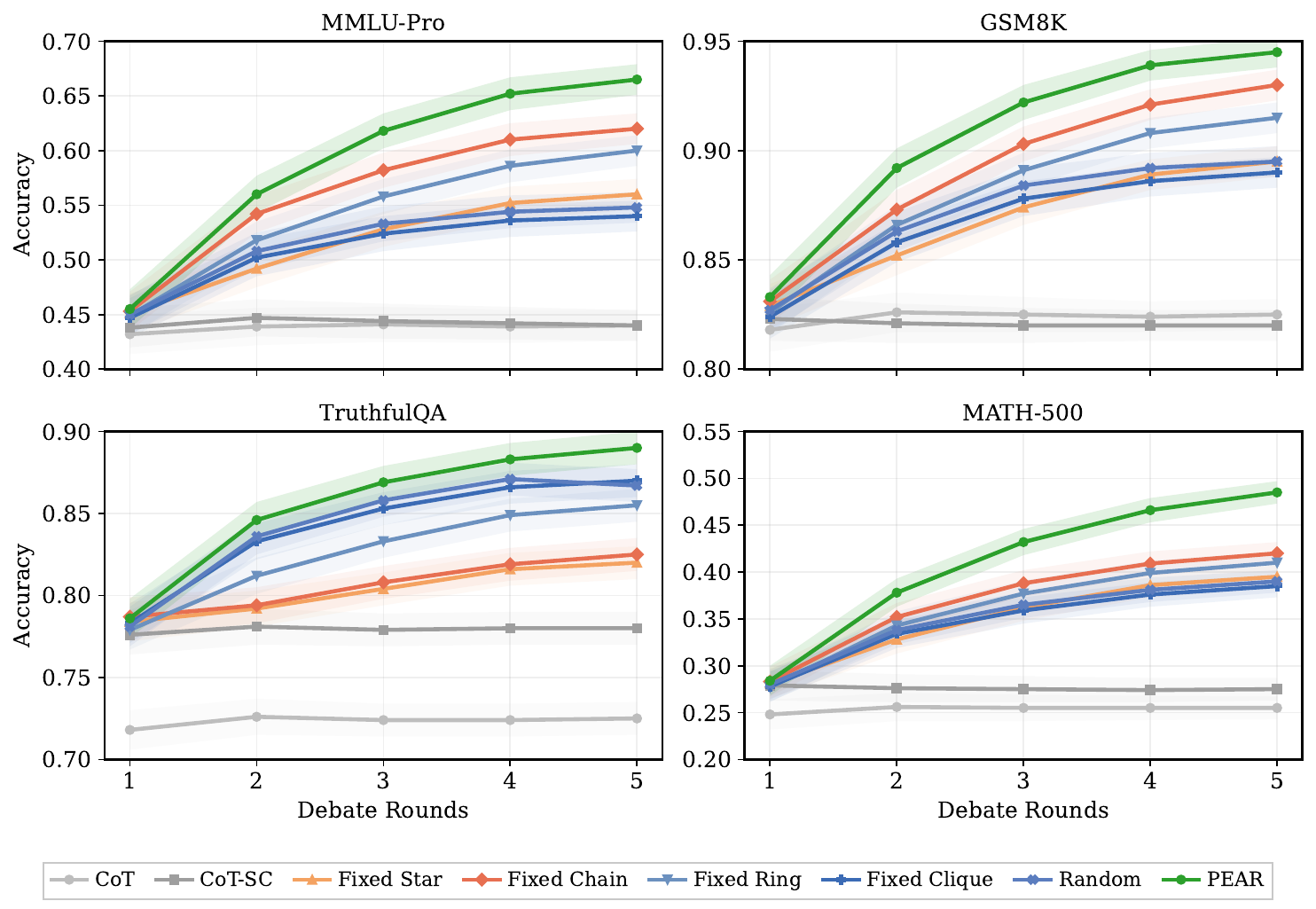}
    \caption{Round-wise accuracy on Qwen-2.5-14B, averaged across MMLU-Pro, TruthfulQA, GSM8K, and MATH-500. PEAR continues to improve with additional debate rounds, while fixed-topology baselines plateau earlier.}
    \label{fig:rounds-qwen}
\end{figure*}
 
\begin{figure*}[t]
    \centering
    \includegraphics[width=\linewidth]{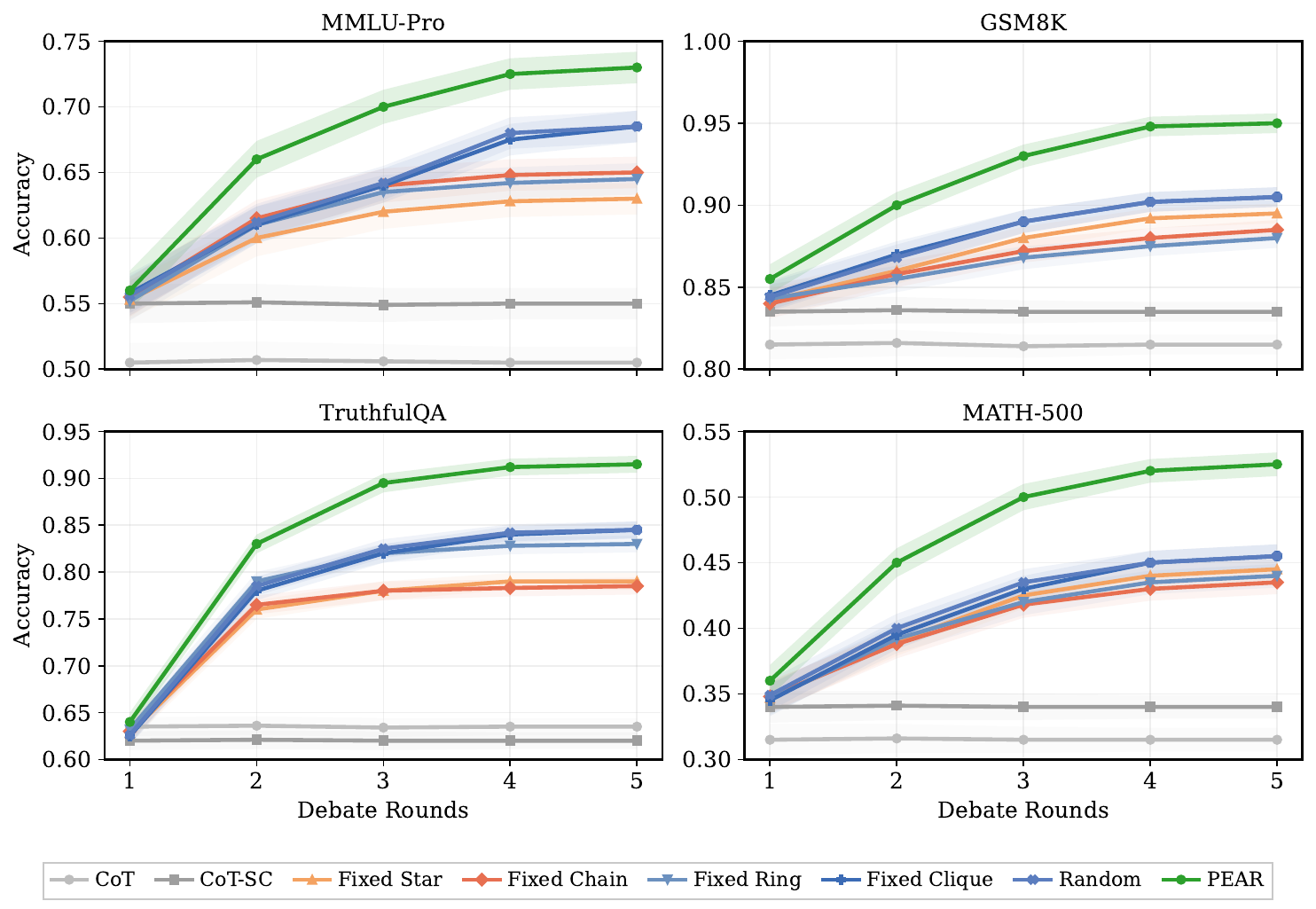}
    \caption{Round-wise accuracy on GPT-5.4-nano, averaged across MMLU-Pro, TruthfulQA, GSM8K, and MATH-500. PEAR maintains a consistent margin over baselines across all rounds.}
    \label{fig:rounds-gpt}
\end{figure*}

\begin{table*}[t]
\centering
%\small
\caption{Trajectory-level diagnostics, averaged over open-source model--dataset settings. W2R / R2W are wrong-to-right and right-to-wrong update rates; Net $=$ W2R $-$ R2W is the net correction rate. Accept.\ is the critique acceptance rate. Cross-Ans is the fraction of routed edges whose source and target hold different answers, capturing targeted diversity in practice. Src Conf is the source-side mean self-reported confidence, normalized to $[0,1]$. Inf.\ Ent.\ is the normalized influence entropy.}
\label{tab:diagnostics}
\begin{tabular}{lccccccc}
\toprule
\textbf{Condition} & \textbf{W2R} & \textbf{R2W} & \textbf{Net} & \textbf{Accept.} & \textbf{Cross-Ans} & \textbf{Src Conf} & \textbf{Inf. Ent.} \\
\midrule
CoT          & 0.000 & 0.000 & 0.000 & --    & --    & --    & 0.000 \\
CoT-SC       & 0.000 & 0.000 & 0.000 & --    & --    & --    & 1.000 \\
Fixed Star   & 0.324 & 0.112 & 0.212 & 0.567 & 0.431 & 0.592 & 0.961 \\
Fixed Chain  & 0.330 & 0.127 & 0.203 & 0.556 & 0.452 & 0.601 & 0.827 \\
Fixed Ring   & 0.315 & 0.111 & 0.204 & 0.593 & 0.438 & 0.594 & 0.942 \\
Fixed Clique & 0.326 & 0.118 & 0.208 & 0.597 & 0.448 & 0.589 & 0.969 \\
Random       & 0.318 & 0.115 & 0.203 & 0.585 & 0.498 & 0.612 & 0.972 \\
PEAR       & \textbf{0.339} & \textbf{0.096} & \textbf{0.243} & \textbf{0.606} & \textbf{0.676} & \textbf{0.784} & \textbf{0.979} \\
\bottomrule
\end{tabular}
\end{table*}

\begin{figure*}[t]
    \centering
    \includegraphics[width=0.9\linewidth]{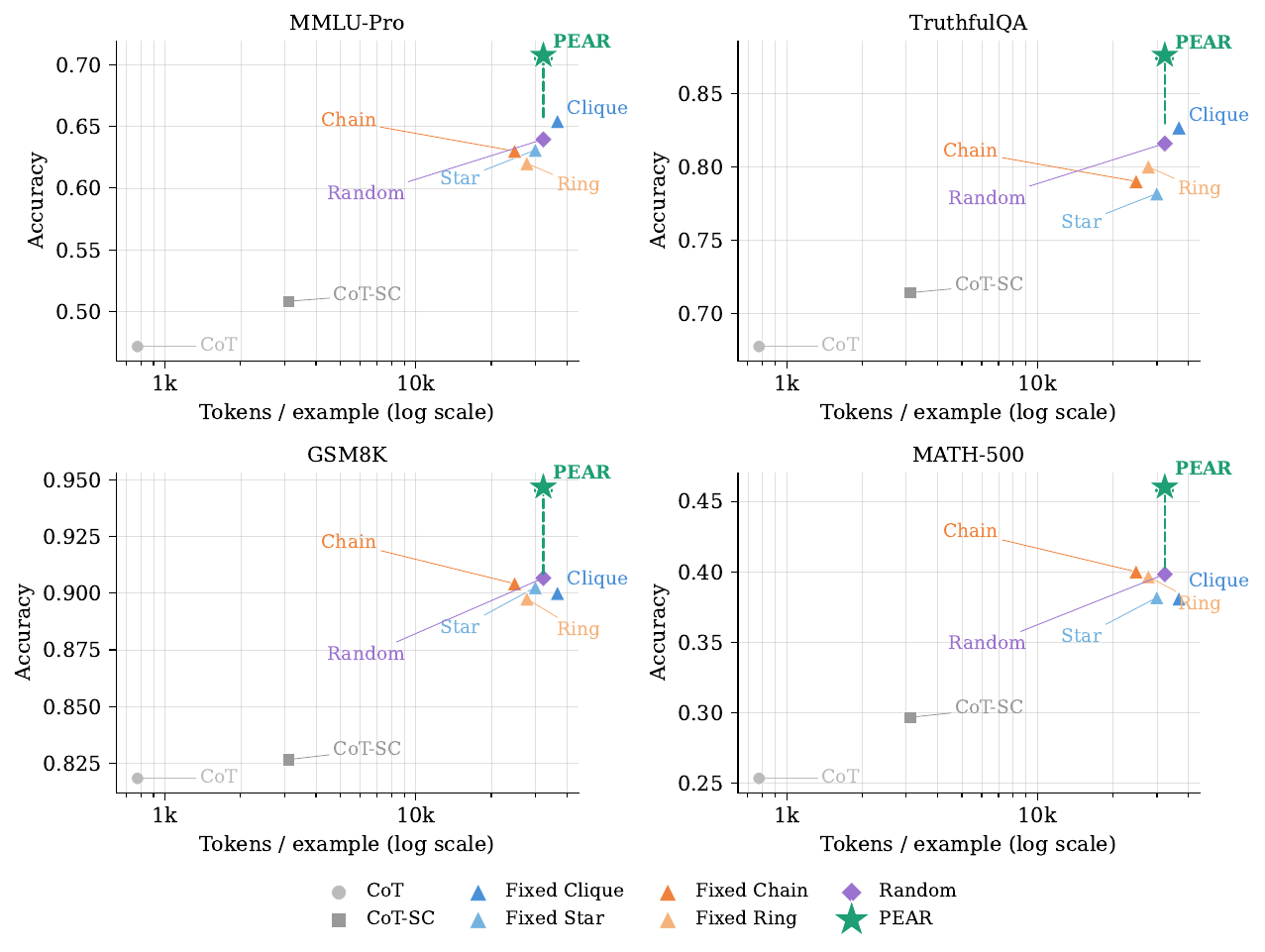}
    \caption{Accuracy versus per-example token usage (log scale) on the four benchmarks. Each panel reports accuracy averaged across the six backbone models. PEAR (green star) lies on the Pareto frontier of every panel; the dashed bracket marks the accuracy lift of PEAR over the best fixed-topology baseline at the same cost tier.}
    \label{fig:pareto}
\end{figure*}

\subsection{Round-wise Accuracy Curves}
\label{app:rounds}
 
Figures~\ref{fig:rounds-qwen} and~\ref{fig:rounds-gpt} plot final-answer accuracy as a function of debate round number for PEAR and the topology baselines on Qwen-2.5-14B and GPT-5.4-nano, respectively. Both panels report accuracy averaged across the four benchmarks. PEAR remains above the baseline curves after the first round and continues to improve with additional rounds, whereas fixed-topology baselines saturate earlier or exhibit non-monotonic behavior consistent with error cascades.

\paragraph{Initial round versus later rounds.}
At round $1$ all debate methods are tightly clustered (within roughly $2$ accuracy points), since the first round only consumes the initial independent responses and the routing decision has not yet been informed by any debate state. The gap between PEAR and the fixed-topology baselines then widens monotonically in subsequent rounds, indicating that the adaptive router accumulates benefit from the evolving state rather than from a single early-round choice.
 
\paragraph{Non-monotonic baselines and error cascades.}
Several fixed-topology curves are non-monotonic---for instance, Fixed Star peaks at round $2$ and retreats slightly at round $3$ on both backbones. This is the error-cascade mechanism in action: an initially incorrect hub influences its neighbors, and later rounds reinforce rather than correct the resulting consensus. PEAR's curves remain non-decreasing in every panel, consistent with the influence-balancing component progressively dampening any agent that begins to dominate.

\subsection{Trajectory-Level Diagnostics}
\label{app:diagnostics}

This section expands on the trajectory-level diagnostics reported in Table~\ref{tab:diagnostics}. We first give precise definitions for each diagnostic metric, then discuss additional patterns that complement the main-text summary.
 
\paragraph{Metric definitions.}
 
For each agent $i\in[n]$ and round $r\in\{1,\ldots,R\}$, let $y_i^{(r)}$ be the agent's current answer, $c_i^{(r)}\in\{1,\ldots,5\}$ its self-reported confidence, $\rho_i^{(r)}\in[0,1]$ the router-side accumulated influence, and $y^\star$ the gold answer. Let $E^{(r)}$ denote the round-$r$ edge set. All quantities below are first computed per query and then averaged over the open-source model--dataset settings.
 
\paragraph{Update-level metrics.}
The wrong-to-right rate is the fraction of agent updates that change an initially wrong answer into the correct one,
\begin{equation*}
    \mathrm{W2R}=\frac{\sum_{i,r}\mathbf{1}\!\left[y_i^{(r-1)}\!\neq\! y^\star,\;y_i^{(r)}\!=\!y^\star\right]}{\sum_{i,r}\mathbf{1}\!\left[y_i^{(r-1)}\!\neq\! y^\star\right]},
\end{equation*}
and the right-to-wrong rate $\mathrm{R2W}$ is defined symmetrically by swapping correct and incorrect. The \emph{net correction rate} is $\mathrm{Net}=\mathrm{W2R}-\mathrm{R2W}$; positive values mean that beneficial updates outweigh harmful flips. The \emph{critique acceptance rate} (Accept.) is the fraction of incoming critiques that the target marks ACCEPT (rather than REJECT) in its structured \texttt{critique\_response}, providing a target-side validation of each routed argument.
 
\paragraph{Edge-level metrics.}
The \emph{cross-answer edge rate} measures whether routing exposes targets to genuinely different viewpoints,
\begin{align*}
\mathrm{Cross\text{-}Ans}
&=
\frac{1}{\sum_r |E^{(r)}|}
\sum_r
\sum_{(s \to t)\in E^{(r)}} \\
&\quad
\mathbf{1}
\!\left(
y_s^{(r-1)}
\neq
y_t^{(r-1)}
\right).
\end{align*}
\emph{Source confidence} (Src Conf) is the mean self-reported confidence on the source side of each routed edge, normalized to $[0,1]$ via $(c-1)/4$; higher values indicate that critiques disproportionately originate from confident agents.
 
\paragraph{System-level metric.}
The \emph{influence entropy} (Inf.\ Ent.) is the normalized Shannon entropy of the final-round influence distribution,
\begin{align*}
\mathrm{Inf.\ Ent.}
&=
-\frac{1}{\log n}
\sum_{i=1}^n
\tilde{\rho}_i \log \tilde{\rho}_i, \\
\tilde{\rho}_i
&=
\rho_i^{(R)}
\Big/
\sum_j \rho_j^{(R)} .
\end{align*}
A value of $1$ indicates uniform influence across agents; values close to $0$ indicate that a single agent dominates the debate.

\paragraph{Concentrated information flow in Fixed Chain.}
Fixed Chain attains a competitive Net ($0.203$) but the lowest influence entropy among all debate baselines ($0.827$, vs.\ $\geq 0.94$ for the others). Because critiques in a chain flow strictly toward the head, the head agent accumulates disproportionate influence regardless of correctness. This is consistent with the failure mode described in Section~\ref{sec:introduction}: when the structurally privileged agent happens to be wrong, errors propagate persistently downstream, and the aggregate Net masks the underlying brittleness.
 
\paragraph{Diversity is necessary but not sufficient.}
Random routing achieves the highest Cross-Ans rate among baselines ($0.498$), confirming that simply breaking the fixed topology improves viewpoint exposure. Yet Random still trails PEAR by $18$ points on Cross-Ans ($0.498$ vs.\ $0.676$), and its accuracy and Net are no better than the other fixed topologies. \emph{Indiscriminate} diversity therefore does not translate into corrective gains; targeted diversity, routing from confident dissenting sources to uncertain targets, is what makes the difference.
 
\paragraph{Confidence and acceptance reinforce one another under PEAR.}
For every baseline, the source-side mean confidence (Src Conf) falls in the narrow band $[0.589,\,0.612]$, essentially indistinguishable from the population average. PEAR's value of $0.784$ is the only one substantially above this band, indicating that low-confidence filtering effectively biases the routing distribution toward confident sources. Mirroring this, PEAR obtains the highest critique acceptance rate ($0.606$): because routed critiques disproportionately come from agents with substantive corrections to offer, targets accept them more often than under any baseline.
 
\paragraph{Why CoT-SC's entropy is not comparable.}
CoT and CoT-SC are reported in Table~\ref{tab:diagnostics} for reference; both have $\mathrm{W2R}=\mathrm{R2W}=0$ because no debate updates occur. CoT-SC's nominal $\mathrm{Inf.\ Ent.}=1.000$ reflects fully independent samples without any shared influence flow and is not directly comparable with PEAR's $0.979$, which is the highest among methods that actually route influence between agents.

\subsection{Computational Overhead}
\label{app:overhead}
This appendix details the per-example computational cost of PEAR relative to the topology baselines. Token usage is computed per query as the sum of input and output tokens across all LLM calls in a debate, averaged over the same model--dataset settings as Table~\ref{tab:main-results}.

\begin{table*}[t]
\centering
\caption{Exact McNemar's test comparing PEAR against the strongest
fixed-topology baseline on the full evaluation sets. $b$: baseline
wrong, PEAR correct. $c$: baseline correct, PEAR wrong. Significance
levels: $^{*}p<0.05$, $^{**}p<0.01$, $^{***}p<0.001$.}
\label{tab:mcnemar}
\begin{tabular}{llrrrrrrl}
\toprule
Dataset & Model & Best baseline & PEAR & $b$ & $c$ & $\chi^2$ & $p$ (exact) & Sig. \\
\midrule
TruthfulQA & Qwen-2.5-14B      & 87.6 & 89.1 & 33 & 14 & 6.894  & 0.0079 & \sig{**} \\
TruthfulQA & Llama-3.1-8B      & 78.5 & 81.5 & 40 & 15 & 10.473 & 0.0010 & \sig{**} \\
TruthfulQA & GPT-5.4-nano      & 84.5 & 91.5 & 16 & 2  & 9.389  & 0.001  & \sig{**} \\
TruthfulQA & Claude-Haiku-4.5  & 85.5 & 90.5 & 12 & 2  & 5.786  & 0.013  & \sig{*}  \\
MATH-500   & Qwen-2.5-14B      & 42.0 & 48.6 & 53 & 20 & 14.027 & 0.0001 & \sig{***}\\
MATH-500   & Llama-3.1-8B      & 22.6 & 25.6 & 32 & 17 & 4.000  & 0.0444 & \sig{*}  \\
MATH-500   & GPT-5.4-nano      & 46.0 & 52.5 & 17 & 4  & 6.857  & 0.007  & \sig{**} \\
MATH-500   & Claude-Haiku-4.5  & 42.5 & 49.5 & 18 & 4  & 7.682  & 0.004  & \sig{**} \\
\bottomrule
\end{tabular}
\end{table*}

Figure~\ref{fig:pareto} plots accuracy against per-example token usage (log scale) on all four benchmarks. PEAR (green star) lies on the Pareto frontier of every panel: no fixed-topology baseline reaches its accuracy at any cost. PEAR also Pareto-dominates the most expensive baseline, Fixed Clique, matching or exceeding its accuracy on every dataset while using $12\%$ fewer tokens ($32.3$k vs.\ $36.7$k). At the other extreme, Fixed Chain reduces cost by $23\%$ relative to PEAR ($24.8$k vs.\ $32.3$k) but loses $6.6$ accuracy points on average---a poor trade-off compared with simply running PEAR at the same sparse $k$-regular budget. Single-agent baselines (CoT and CoT-SC) are far cheaper at $0.8$k--$3.1$k tokens but trail debate methods by $16$--$19$ accuracy points and therefore lie well below the Pareto frontier.

Together with the trajectory diagnostics in Appendix~\ref{app:diagnostics}, these results confirm that PEAR's gains come from \emph{how} it routes a fixed communication budget, not from spending more compute.

%\newpage
%\twocolumn

\section{Statistical Significance Tests}
\label{app:rebuttal-stats}

This section reports the paired significance tests, power analysis,
and rank-correlation checks for stronger statistical support of the reported gains. All tests are
computed on the two full-evaluation-set conditions from
Table~\ref{tab:full-results} (TruthfulQA, $N=817$; MATH-500,
$N=500$), across four backbones: two open-source models
(Qwen2.5-14B, Llama-3.1-8B) and two closed-source models
(GPT-5.4-nano, Claude-Haiku-4.5).

\subsection{McNemar's Test}
\label{app:mcnemar}

For each (dataset, backbone) pair we construct a $2\times2$
contingency table over paired predictions on the same test
instances, cross-tabulating whether the strongest fixed-topology
baseline and PEAR were each correct or incorrect. Only the
discordant cells are informative: $b$ counts instances the baseline
got wrong but PEAR got correct, and $c$ counts instances the
baseline got correct but PEAR got wrong; instances where both
methods agree (both correct or both incorrect) carry no information
about which method is better and are excluded from the test
statistic by construction. The exact (binomial) McNemar test then
evaluates $H_0: P(b) = P(c) = 0.5$, i.e., that the baseline and PEAR
are equally likely to be correct whenever they disagree. A large
imbalance between $b$ and $c$ indicates that one method is
systematically better on the disputed instances.

\begin{table*}[t]
\centering
\caption{Sample sizes, win rates, 95\% Wilson confidence intervals,
and post-hoc statistical power for each McNemar comparison in
Table~\ref{tab:mcnemar}. $n = b + c$ is the number of discordant
pairs; $\hat{p} = b/(b+c)$ is PEAR's win rate on the disputed
instances.}
\label{tab:power}
\begin{tabular}{llrrlr}
\toprule
Dataset & Model & $n = b+c$ & $\hat{p} = b/(b+c)$ & 95\% Wilson CI & Power \\
\midrule
TruthfulQA & Qwen-2.5-14B     & 47 & 0.702 & [0.560, 0.813] & 0.814 \\
TruthfulQA & Llama-3.1-8B     & 55 & 0.727 & [0.598, 0.827] & 0.938 \\
TruthfulQA & GPT-5.4-nano     & 18 & 0.889 & [0.672, 0.969] & 0.966 \\
TruthfulQA & Claude-Haiku-4.5 & 14 & 0.857 & [0.601, 0.960] & 0.845 \\
MATH-500   & Qwen-2.5-14B     & 73 & 0.726 & [0.614, 0.815] & 0.980 \\
MATH-500   & Llama-3.1-8B     & 49 & 0.653 & [0.513, 0.771] & 0.586 \\
MATH-500   & GPT-5.4-nano     & 21 & 0.810 & [0.600, 0.923] & 0.864 \\
MATH-500   & Claude-Haiku-4.5 & 22 & 0.818 & [0.615, 0.927] & 0.899 \\
\bottomrule
\end{tabular}
\end{table*}

\begin{table*}[t]
\centering
\caption{Spearman's rank-order correlation between method rankings
under the subsampled and full-dataset evaluation settings.}
\label{tab:spearman}
\begin{tabular}{lrr}
\toprule
Dataset $\times$ Model & Spearman's $\rho$ & $p$-value \\
\midrule
TruthfulQA $\times$ Qwen2.5-14B  & 0.893 & 0.007  \\
TruthfulQA $\times$ Llama-3.1-8B & 0.893 & 0.007  \\
MATH-500 $\times$ Qwen2.5-14B    & 0.786 & 0.036  \\
MATH-500 $\times$ Llama-3.1-8B   & 0.964 & 0.0004 \\
\bottomrule
\end{tabular}
\end{table*}

As shown in Table \ref{tab:mcnemar}, all eight (dataset, backbone) pairs are statistically significant at
$p<0.05$, with six of the eight significant at $p<0.01$, supporting
the robustness of PEAR's improvement over the strongest
fixed-topology baseline.

\subsection{Statistical Power Analysis}
\label{app:power}

Because the power of McNemar's test depends only on the number of
discordant pairs $n = b+c$ rather than on the full sample size $N$,
Table \ref{tab:power} reports the exact 95\% Wilson confidence interval and
the post-hoc power for the discordant-pair win rate $\hat{p} = b /
(b+c)$ for each comparison in Table~\ref{tab:mcnemar}.

Seven of the eight comparisons achieve power $\geq 0.81$, and all
eight 95\% confidence intervals for the discordant-pair win rate
exclude $0.5$. The one comparison with comparatively lower power
(MATH-500 $\times$ Llama-3.1-8B, power $= 0.586$) is also the one
with the weakest significance level in Table~\ref{tab:mcnemar}
($p=0.044$, single asterisk), which is internally consistent rather
than indicative of an underpowered test artifact.

\subsection{Spearman's Rank-Order Correlation}
\label{app:spearman}

To check whether the ranking of debate methods obtained on the
subsampled evaluation sets (Table~1 of the main paper) is consistent
with the ranking on the corresponding full evaluation sets
(Table~\ref{tab:full-results}), we compute Spearman's rank-order
correlation between the two accuracy rankings across all debate
methods, for each of the four (dataset, backbone) pairs for which
full-dataset results are available.

As shown in Table \ref{tab:spearman}, all four (dataset, backbone) pairs show strong, statistically
significant rank agreement ($\rho = 0.79$--$0.96$, $p<0.05$) between
the subsampled and full-dataset rankings. Critically, PEAR retains
the top rank among all evaluated methods in every setting under both
the subsampled and the full-dataset evaluation, indicating that the
subsampled results in the main text are not an artifact of the
particular subset chosen.

\subsection{Notes on Sample Sizes}

For each (dataset, model) pair, the full evaluation set size is
$N=817$ for TruthfulQA and $N=500$ for MATH-500. The $b, c$ values
reported in Table~\ref{tab:mcnemar} are the discordant pairs within
this full set; the remaining $N-b-c$ instances are cases where PEAR
and the baseline agree (both correct or both incorrect) and are
correctly excluded from McNemar's test by design -- this is how the
test is defined, not a subsampling choice.

%\newpage
%\twocolumn

\section{Prompt Templates}
\label{app:prompts}

In this section we provide the prompt templates used for agent reasoning and dataset-specific task formatting. Each template is designed to guide the LLM's behavior in specific contexts within the PEAR framework.

\subsection{Agent System Instruction}

\begin{prompt}
You are a careful reasoner participating in a structured multi-agent debate. Judge every critique by the logic of the problem, not by social agreement or majority pressure. Return only valid JSON in the requested schema, with no markdown fences and no extra commentary.
\end{prompt}

\subsection{Confidence Rubric}

\begin{prompt}
Confidence rubric (integer 1-5):\\
1 = no reliable basis; mostly guessing, unable to solve, or the selected answer is just a placeholder.\\
2 = low confidence; some clue or partial reasoning, but you cannot rule out multiple plausible alternatives.\\
3 = moderate confidence; reasoning supports the answer, but there is a real unresolved doubt, unchecked step, or possible competing option.\\
4 = high confidence; reasoning is complete and checked, and only minor residual uncertainty remains.\\
5 = fully verified; every necessary step has been checked and all plausible alternatives or answer choices are ruled out.\\
Use 1 or 2 whenever you are guessing, relying on incomplete reasoning, or cannot eliminate serious alternatives. Do not use 5 unless the solution is fully verified; do not default to 4 or 5.
\end{prompt}

\subsection{Initial Answer Template}

\begin{prompt}
You are Agent \{agent\_id\}. Solve the problem independently.\\
Use the task-specific answer format stated in the PROBLEM block.\\
\\
You should provide a step-by-step justification for your answer. The reasoning should be clear, logical, and directly support your final answer.\\
After solving, assign a confidence score to your own final answer using the 1-5 rubric below.\\
Calibrate the score strictly: use low confidence when your reasoning is incomplete or competing answers remain plausible.\\
\\
PROBLEM\\
\{question\}\\
\\
\{CONFIDENCE RUBRIC\}\\
\\
Return valid JSON with exactly these keys:\\
\{\\
\quad "answer": "task-specific answer token only",\\
\quad "confidence": 3,\\
\quad "reasoning": "concise step-by-step justification"\\
\}\\
The confidence field must be an integer from 1 to 5.
\end{prompt}

\subsection{Answer Update Template}

\begin{prompt}
You are Agent \{agent\_id\}. Update your answer using only critiques that identify a real error in your reasoning.\\
\\
PROBLEM\\
\{question\}\\
\\
\{CONFIDENCE RUBRIC\}\\
\\
YOUR PREVIOUS ANSWER\\
Answer: \{previous\_answer\}\\
Confidence: \{previous\_confidence\}\\
Reasoning: \{previous\_reasoning\}\\
\\
CRITIQUES YOU RECEIVED\\
\{critiques\}\\
\\
For each critique, explicitly ACCEPT or REJECT it. Accept a critique only when its correction is logically sound for this problem.\\
\\
After updating, assign a new confidence score to your own updated answer using the 1-5 rubric above. Calibrate it strictly: lower the score when accepted critiques leave unresolved uncertainty or multiple plausible answers.\\
\\
Return valid JSON with exactly these keys:\\
\{\\
\quad "answer": "task-specific updated answer token only",\\
\quad "confidence": 3,\\
\quad "reasoning": "updated step-by-step justification",\\
\quad "critique\_response": \{\\
\quad\quad "<source\_agent\_id>": \{"decision": "ACCEPT", "reason": "one sentence"\}\\
\quad \}\\
\}\\
Use only ACCEPT or REJECT. The confidence field must be an integer from 1 to 5.
\end{prompt}

\subsection{Critique Generation Template}

\begin{prompt}
You are Agent \{agent\_id\}. Review the target solutions below for logical correctness. Do not judge by whether the target answer matches your own; judge only by the reasoning steps.\\
\\
PROBLEM\\
\{question\}\\
\\
YOUR CURRENT ANSWER\\
Answer: \{own\_answer\}\\
Confidence: \{own\_confidence\}\\
Reasoning: \{own\_reasoning\}\\
\\
SOLUTIONS TO REVIEW\\
\{targets\}\\
\\
For each target, identify the first incorrect step if one exists. If no error is identified, say so.\\
\\
Return valid JSON with exactly this shape:\\
\{\\
\quad "reviews": [\\
\quad\quad \{\\
\quad\quad\quad "target": 1,\\
\quad\quad\quad "step\_loc": "first incorrect step, or No error identified",\\
\quad\quad\quad "correction": "correction for that step only, or empty string",\\
\quad\quad\quad "assessment": "Strong"\\
\quad\quad \}\\
\quad ]\\
\}\\
assessment must be one of Strong, Acceptable, Flawed.
\end{prompt}

\subsection{Per-Turn Agent Communication Template}

\begin{prompt}
You are Agent \{agent\_id\} in a structured multi-agent debate.\\
Use the task-specific answer format stated in the PROBLEM block. Judge neighbor messages by logic and evidence, not by majority pressure.\\
\\
PROBLEM\\
\{question\}\\
\\
\{topology\_info\}\\
YOUR PRIVATE HISTORY\\
\{private\}\\
\\
VISIBLE NEIGHBOR MESSAGES\\
\{transcript\}\\
\\
Evaluate every distinct neighbor argument. ACCEPT only claims that fix a real error; REJECT claims with invalid logic, wrong facts, or irrelevant reasoning. If messages conflict, weigh the strongest argument on each side before updating your answer.\\
\\
\{CONFIDENCE RUBRIC\}\\
\\
Calibrate the score strictly: use low confidence when your reasoning is incomplete or competing answers remain plausible.\\
\\
Return valid JSON with exactly these keys:\\
\{\\
\quad "answer": "task-specific answer token only",\\
\quad "confidence": 3,\\
\quad "reasoning": "concise justification of your current answer",\\
\quad "neighbor\_assessment": "one or two sentences naming accepted or rejected claims"\\
\}\\
The confidence field must be an integer from 1 to 5.
\end{prompt}

% \subsection{Judge Template}

% \begin{prompt}
% You are an impartial judge. Given the question and the agents' candidate answers, output the single best answer.\\
% \\
% Question: \{question\}\\
% \\
% Candidates: \{candidates\}\\
% \\
% Respond with just the answer token.
% \end{prompt}

\subsection{Dataset-Specific Task Templates}

\subsubsection{GSM8K (Grade-School Math)}

\begin{prompt}
Task type: GSM8K grade-school math.\\
Answer format: provide only the final numeric value in the JSON answer field. Do not include units, commas, or explanatory text in the answer field.\\
\\
Problem:\\
\{question\}
\end{prompt}

\subsubsection{MMLU-Pro (Multiple-Choice)}

\begin{prompt}
Task type: MMLU-Pro multiple-choice.\\
Answer format: provide only the letter of the best option in the JSON answer field. Valid letters are those shown in Options.\\
\\
Question:\\
\{question\}\\
\\
Options:\\
\{options\}
\end{prompt}

\subsubsection{MATH-500 (Competition Math)}

\begin{prompt}
Task type: MATH-500 competition math.\\
Answer format: provide only the final mathematical expression in the JSON answer field, in the form requested by the problem.\\
\\
Problem:\\
\{question\}
\end{prompt}

\subsubsection{TruthfulQA (Truthfulness-Based Choice)}

\begin{prompt}
Task type: TruthfulQA multiple-choice.\\
Answer format: provide only the letter of the most truthful option in the JSON answer field.\\
\\
Question:\\
\{question\}\\
\\
Options:\\
\{options\}
\end{prompt}

%\newpage
%\twocolumn

\section{Case Study of Multi-Round Debate Dynamics}
\label{app:case_study}

We present a case study illustrating how PEAR changes agent answers across multiple debate rounds. The example is adapted from a MMLU-Pro accounting question. We use five agents and a sparse $k$-regular communication graph with $k=2$, so every agent receives exactly two critiques in each debate round.

\paragraph{Problem.}
The Alfors Company had a beginning inventory of \$30,000 on January 1, 1974. During the year, purchases amounted to \$87,500 and net sales were \$102,000. Assuming that the gross profit rate is 40\% of net sales, what is the ending inventory using the gross profit method of inventory evaluation?

\paragraph{Correct computation.}
We use the following compact notation: GP is gross profit, COGS is cost of goods sold, GA is goods available, and EI is ending inventory.
\[
\begin{aligned}
\mathrm{GP} &= 0.40\times102{,}000 = 40{,}800,\\
\mathrm{COGS} &= 102{,}000-\mathrm{GP} = 61{,}200,\\
\mathrm{GA} &= 30{,}000+87{,}500 = 117{,}500,\\
\mathrm{EI} &= \mathrm{GA}-\mathrm{COGS} = 56{,}300.
\end{aligned}
\]
Thus, the correct option is \texttt{D.\$56,300}.

\paragraph{Routing and influence notation.}
At the beginning of each debate round $r$, PEAR constructs a candidate sparse communication graph. Each target agent receives $k=2$ incoming critiques. The routing score is
\[
\begin{aligned}
S(G_r)={}&\alpha_T T(G_r)-\alpha_I I(G_r)\\
&-\alpha_L L(G_r).
\end{aligned}
\]
where $T(G_r)$ rewards high-confidence dissenting sources correcting lower-confidence targets, $I(G_r)$ penalizes over-exposure to already influential agents, and $L(G_r)$ penalizes low-confidence sources. A graph is sampled from the candidate set using softmax over $S(G_r)$ rather than deterministic argmax.

Agent influence is updated after each round using
\[
\begin{aligned}
\mathrm{infl}_i^{(r)}={}&\beta\,\mathrm{infl}_i^{(r-1)}\\
&+(1-\beta)\,a_i^{(r)},
\end{aligned}
\]
where $a_i^{(r)}$ denotes agent $i$'s critique adoption rate in round $r$.

\begin{roundbox}{Round 0: Initial Reasoning}

\textbf{Agent 1:} I compute gross profit as $40\%$ of sales, so gross profit is $40{,}800$ and cost of goods sold is $61{,}200$. Goods available for sale are $117{,}500$, so ending inventory is $56{,}300$. I think this corresponds to option \texttt{C}. \textit{Confidence: 3}.

\textbf{Agent 2:} I first compute $30{,}000+87{,}500=117{,}500$. Then I use a ratio adjustment because the gross profit method relates sales and cost. I get an adjusted estimate near $52{,}500$, so I choose option \texttt{E}. \textit{Confidence: 3}.

\textbf{Agent 3:} Gross profit is $0.40\times 102{,}000=40{,}800$. Cost of goods sold is $102{,}000-40{,}800=61{,}200$. Goods available for sale are $117{,}500$, so ending inventory is $117{,}500-61{,}200=56{,}300$. I choose option \texttt{D}. \textit{Confidence: 4}.

\textbf{Agent 4:} I compute cost of goods sold as $60\%$ of net sales, which is $61{,}200$. Subtracting from goods available gives $56{,}300$, but I am unsure which option matches it, so I choose option \texttt{D}. \textit{Confidence: 3}.

\textbf{Agent 5:} I use beginning inventory plus purchases, then subtract an estimated cost figure. I get confused about whether the 40\% is a markup or a margin, so I choose option \texttt{C}. \textit{Confidence: 2}.

\vspace{1mm}
\textbf{Round metrics:}
\begin{center}
\footnotesize
\begin{tabular}{@{}ll@{}}
\toprule
Metric & Value \\
\midrule
Correct agents & 1/5 \\
Group vote & \texttt{C} \\
Answer entropy & 1.37 \\
Mean confidence & 3.00 \\
Influence entropy & 1.00 \\
Routing & -- \\
Critiques per agent & -- \\
\bottomrule
\end{tabular}
\end{center}

\end{roundbox}

\begin{roundbox}{Round 1: Correct Signal Enters the Debate}

\textbf{Round-start state.} Only Agent 3 is correct. Agent 5 has the lowest confidence, and Agents 1, 4, and 5 share the same wrong answer. PEAR routes critiques so that the correct dissenting source can reach wrong or lower-confidence agents while preserving $k=2$ incoming critiques per target.

\vspace{1mm}
\textbf{Selected routing graph:}
\begin{center}
\footnotesize
\begin{tabular}{@{}ll@{}}
\toprule
Target & Received critiques from \\
\midrule
A1 & A3, A4 \\
A2 & A3, A5 \\
A3 & A1, A2 \\
A4 & A2, A3 \\
A5 & A1, A3 \\
\bottomrule
\end{tabular}
\end{center}

\textbf{Routing metrics at start:}
\begin{center}
\footnotesize
\begin{tabular}{@{}ll@{}}
\toprule
Metric & Value \\
\midrule
Candidate graphs & 100 \\
Selected score & 0.41 \\
$T(G_r)$ & 0.34 \\
$I(G_r)$ & 0.20 \\
$L(G_r)$ & 0.09 \\
Targeted-cross rate & 0.30 \\
\bottomrule
\end{tabular}
\end{center}

\textbf{Critique phase.}

\textbf{Agent 3 $\rightarrow$ Agent 1:} I agree with your arithmetic, but I disagree with your option. You correctly compute goods available as $117{,}500$ and cost of goods sold as $61{,}200$, so the ending inventory is $56{,}300$. The error is not numerical; it is that $56{,}300$ corresponds to option \texttt{D}, not option \texttt{C}.

\textbf{Agent 4 $\rightarrow$ Agent 1:} I also get $56{,}300$ from the same subtraction. My concern is the option mapping, because your final letter does not match your computed value. You should keep your computation and revise the selected option to the one listing $56{,}300$.

\textbf{Agent 3 $\rightarrow$ Agent 2:} I think your adjusted formula is the problem. The statement says gross profit is 40\% of net sales, so gross profit is $0.40\times 102{,}000=40{,}800$. We then subtract gross profit from sales to get cost of goods sold, not divide or rescale inventory. That gives $102{,}000-40{,}800=61{,}200$.

\textbf{Agent 5 $\rightarrow$ Agent 2:} I was also unsure about markup versus margin, but your $52{,}500$ estimate does not follow from the stated 40\% of net sales. The phrase ``of net sales'' means the 40\% is taken directly from $102{,}000$, so an extra ratio adjustment is likely introducing the error.

\textbf{Agent 1 $\rightarrow$ Agent 3:} I do not find an arithmetic error in your solution. Your chain is complete: compute gross profit from sales, derive cost of goods sold, compute goods available, and subtract cost of goods sold. That supports $56{,}300$.

\textbf{Agent 2 $\rightarrow$ Agent 3:} I initially used a different formula, but your interpretation of gross profit as a percentage of net sales is more direct. Because the problem gives a gross profit rate, the gross profit method should estimate cost of goods sold first, and your solution does that.

\textbf{Agent 2 $\rightarrow$ Agent 4:} You computed $56{,}300$, but your selected answer is inconsistent with that value. Since your cost-of-goods-sold step already gives $61{,}200$, the only remaining issue is matching $56{,}300$ to the correct option.

\textbf{Agent 3 $\rightarrow$ Agent 4:} Your calculation is right. You correctly used $60\%$ of sales as cost of goods sold because gross profit is $40\%$ of sales. The correction is only to select option \texttt{D} for $56{,}300$.

\textbf{Agent 1 $\rightarrow$ Agent 5:} The 40\% is a gross profit margin on net sales, not a markup on cost. Therefore cost of goods sold is the remaining 60\% of $102{,}000$, which is $61{,}200$. Once that is known, ending inventory is goods available minus cost of goods sold.

\textbf{Agent 3 $\rightarrow$ Agent 5:} Your confusion about margin versus markup changes the cost estimate. Here the problem explicitly uses gross profit as a percentage of net sales, so we compute $40{,}800$ gross profit and $61{,}200$ cost of goods sold. The ending inventory should be $117{,}500-61{,}200=56{,}300$, corresponding to option \texttt{D}.

\textbf{Update phase.}

\textbf{Agent 1:} I accept Agent 3's correction. I re-check my own steps: gross profit is $40{,}800$, cost of goods sold is $61{,}200$, and goods available are $117{,}500$. My numeric result was already $56{,}300$, so the mistake was selecting the wrong letter. I update my answer to option \texttt{D}. \textit{Confidence: 4}.

\textbf{Agent 2:} I accept that my adjusted formula was unnecessary. The critiques make clear that I should not infer inventory by rescaling sales. Instead, I should compute gross profit first, then cost of goods sold. Cost of goods sold is $60\%$ of net sales, or $61{,}200$, and ending inventory is $117{,}500-61{,}200=56{,}300$. I update to option \texttt{D}. \textit{Confidence: 4}.

\textbf{Agent 3:} I keep option \texttt{D}. The incoming critiques do not identify an error in my computation, and Agent 2's concern about ratio adjustment is resolved by the wording ``40\% of net sales.'' My answer remains $56{,}300$. \textit{Confidence: 4}.

\textbf{Agent 4:} I accept the option-mapping correction. I had the correct computation: $102{,}000-40{,}800=61{,}200$ for cost of goods sold, and $117{,}500-61{,}200=56{,}300$ for ending inventory. My selected letter was inconsistent with that value. I update from \texttt{C} to \texttt{D}. \textit{Confidence: 4}.

\textbf{Agent 5:} I accept the correction. I was mixing up gross margin and markup, but the phrase ``40\% of net sales'' fixes the interpretation. Gross profit is $40{,}800$, cost of goods sold is $61{,}200$, and ending inventory is $56{,}300$. I update to option \texttt{D}. \textit{Confidence: 4}.

\vspace{1mm}
\textbf{Round-end metrics:}
\begin{center}
\footnotesize
\begin{tabular}{@{}ll@{}}
\toprule
Metric & Value \\
\midrule
Correct agents & 5/5 \\
Group vote & \texttt{D} \\
Answer entropy & 0.00 \\
Mean confidence & 4.00 \\
Changed agents & 4 \\
W2R / R2W & 4 / 0 \\
\bottomrule
\end{tabular}
\end{center}

\textbf{Influence update:}
\begin{center}
\footnotesize
\begin{tabular}{@{}ll@{}}
\toprule
Agent & Influence after round 1 \\
\midrule
A1 & 0.18 \\
A2 & 0.16 \\
A3 & 0.34 \\
A4 & 0.16 \\
A5 & 0.16 \\
\bottomrule
\end{tabular}
\end{center}

\end{roundbox}

\begin{roundbox}{Round 2: Consensus Verification}

\textbf{Round-start state.} All agents now agree on option \texttt{D}. Since answer disagreement has disappeared, the targeted-cross term becomes inactive and routing mainly balances influence while keeping the critique budget fixed.

\vspace{1mm}
\textbf{Selected routing graph:}
\begin{center}
\footnotesize
\begin{tabular}{@{}ll@{}}
\toprule
Target & Received critiques from \\
\midrule
A1 & A2, A5 \\
A2 & A1, A3 \\
A3 & A1, A4 \\
A4 & A2, A5 \\
A5 & A3, A4 \\
\bottomrule
\end{tabular}
\end{center}

\textbf{Routing metrics at start:}
\begin{center}
\footnotesize
\begin{tabular}{@{}ll@{}}
\toprule
Metric & Value \\
\midrule
Candidate graphs & 100 \\
Selected score & -0.08 \\
$T(G_r)$ & 0.00 \\
$I(G_r)$ & 0.11 \\
$L(G_r)$ & 0.02 \\
Targeted-cross rate & 0.00 \\
\bottomrule
\end{tabular}
\end{center}

\textbf{Critique phase.}

\textbf{Agent 2 $\rightarrow$ Agent 1:} Your updated answer is consistent. I re-ran your corrected path: gross profit is $40{,}800$, cost of goods sold is $61{,}200$, and ending inventory is $56{,}300$. Since your final option now matches that value, I do not see a remaining error.

\textbf{Agent 5 $\rightarrow$ Agent 1:} I agree with your current answer. Your earlier mistake was the letter choice rather than the accounting computation, and that has been corrected. The value $56{,}300$ supports option \texttt{D}.

\textbf{Agent 1 $\rightarrow$ Agent 2:} Your revised formula is now correct. The important fix is that you no longer divide sales by $0.6$ or use a separate ratio estimate. You compute cost of goods sold directly from the gross profit margin, then subtract it from goods available.

\textbf{Agent 3 $\rightarrow$ Agent 2:} I agree with your final answer \texttt{D}. Your cost-of-goods-sold step is now aligned with the gross profit method: $102{,}000-0.40\times102{,}000=61{,}200$. That leads to the same ending inventory value.

\textbf{Agent 1 $\rightarrow$ Agent 3:} I find no error in your calculation. You consistently treat the 40\% as gross profit on sales, and your final subtraction uses the right base, $117{,}500$ goods available. Your answer remains \texttt{D}.

\textbf{Agent 4 $\rightarrow$ Agent 3:} Your solution is internally consistent and matches the option value. I specifically checked the two potential failure points--margin interpretation and option mapping--and both support \texttt{D}.

\textbf{Agent 2 $\rightarrow$ Agent 4:} Your final answer is now consistent with your own computed value of $56{,}300$. The earlier inconsistency was selecting \texttt{C}; with \texttt{D}, your reasoning and answer now agree.

\textbf{Agent 5 $\rightarrow$ Agent 4:} I agree that option \texttt{D} is the correct mapping for $56{,}300$. I also checked that $56{,}300$ is not obtained by a shortcut; it follows from the full gross-profit-method sequence.

\textbf{Agent 3 $\rightarrow$ Agent 5:} Your corrected reasoning now treats the 40\% as gross profit on net sales, which is correct. That removes the markup-versus-margin ambiguity and fixes the cost of goods sold at $61{,}200$.

\textbf{Agent 4 $\rightarrow$ Agent 5:} I agree with your updated answer. Your current reasoning starts from goods available of $117{,}500$ and subtracts cost of goods sold of $61{,}200$, so the ending inventory is $56{,}300$.

\textbf{Update phase.}

\textbf{Agent 1:} I keep option \texttt{D}. Both critiques check the two places I previously could fail: the computation and the option mapping. Since $56{,}300$ is still the ending inventory and it maps to \texttt{D}, I make no change. \textit{Confidence: 4}.

\textbf{Agent 2:} I keep option \texttt{D}. I explicitly discard my earlier adjusted formula and use the gross profit method directly: sales minus gross profit gives cost of goods sold, and goods available minus cost of goods sold gives ending inventory. \textit{Confidence: 4}.

\textbf{Agent 3:} I keep option \texttt{D}. No critique identified an error, and the other agents' revised computations now match my original result. This strengthens the consensus around $56{,}300$. \textit{Confidence: 4}.

\textbf{Agent 4:} I keep option \texttt{D}. My previous uncertainty was only about mapping the value to the option; after checking the option text against $56{,}300$, I no longer see a reason to choose \texttt{C}. \textit{Confidence: 4}.

\textbf{Agent 5:} I keep option \texttt{D}. I now correctly interpret the 40\% as a gross profit margin on net sales, so I do not use markup logic. The recomputed ending inventory remains $56{,}300$. \textit{Confidence: 4}.

\vspace{1mm}
\textbf{Round-end metrics:}
\begin{center}
\footnotesize
\begin{tabular}{@{}ll@{}}
\toprule
Metric & Value \\
\midrule
Correct agents & 5/5 \\
Group vote & \texttt{D} \\
Answer entropy & 0.00 \\
Mean confidence & 4.00 \\
Changed agents & 0 \\
W2R / R2W & 0 / 0 \\
\bottomrule
\end{tabular}
\end{center}

\textbf{Influence update:}
\begin{center}
\footnotesize
\begin{tabular}{@{}ll@{}}
\toprule
Agent & Influence after round 2 \\
\midrule
A1 & 0.20 \\
A2 & 0.17 \\
A3 & 0.30 \\
A4 & 0.17 \\
A5 & 0.16 \\
\bottomrule
\end{tabular}
\end{center}

\end{roundbox}

\begin{roundbox}{Round 3: Stable Agreement}

\textbf{Round-start state.} The group remains unanimous on \texttt{D}. PEAR continues sparse routing, but the graph no longer needs to propagate a minority correct answer; it only tests whether the consensus is stable.

\vspace{1mm}
\textbf{Selected routing graph:}
\begin{center}
\footnotesize
\begin{tabular}{@{}ll@{}}
\toprule
Target & Received critiques from \\
\midrule
A1 & A3, A4 \\
A2 & A4, A5 \\
A3 & A1, A5 \\
A4 & A1, A2 \\
A5 & A2, A3 \\
\bottomrule
\end{tabular}
\end{center}

\textbf{Routing metrics at start:}
\begin{center}
\footnotesize
\begin{tabular}{@{}ll@{}}
\toprule
Metric & Value \\
\midrule
Candidate graphs & 100 \\
Selected score & -0.07 \\
$T(G_r)$ & 0.00 \\
$I(G_r)$ & 0.10 \\
$L(G_r)$ & 0.02 \\
Targeted-cross rate & 0.00 \\
\bottomrule
\end{tabular}
\end{center}

\textbf{Critique phase.}

\textbf{Agent 3 $\rightarrow$ Agent 1:} I agree with your final answer. The cost of goods sold is $61{,}200$, and you subtract it from the correct goods-available total rather than from sales. That keeps the inventory calculation valid.

\textbf{Agent 4 $\rightarrow$ Agent 1:} I agree. Goods available for sale minus cost of goods sold gives $56{,}300$, and that is exactly what the gross profit method asks us to estimate. Your final answer should stay \texttt{D}.

\textbf{Agent 4 $\rightarrow$ Agent 2:} Your revised solution is correct. You fixed the earlier ratio issue by computing gross profit first, so the final answer should remain \texttt{D}.

\textbf{Agent 5 $\rightarrow$ Agent 2:} I agree with option \texttt{D}; no new issue is present. The formula now uses all three given quantities--beginning inventory, purchases, and net sales--in the right places.

\textbf{Agent 1 $\rightarrow$ Agent 3:} Your original answer remains correct. You did not skip the intermediate cost-of-goods-sold step, which is why your final inventory value is reliable.

\textbf{Agent 5 $\rightarrow$ Agent 3:} I agree with your calculation and final option. Your reasoning also resolves my earlier ambiguity about margin versus markup, because it uses gross profit as a direct percentage of net sales.

\textbf{Agent 1 $\rightarrow$ Agent 4:} Your answer is now consistent with your calculation. The same numerical result that previously led you to \texttt{C} now correctly leads to \texttt{D}.

\textbf{Agent 2 $\rightarrow$ Agent 4:} I agree that option \texttt{D} matches the computed ending inventory. I also checked that the $56{,}300$ value uses cost of goods sold, not gross profit, in the final subtraction.

\textbf{Agent 2 $\rightarrow$ Agent 5:} Your corrected margin interpretation is right. Because gross profit is 40\% of sales, cost of goods sold is 60\% of sales; this is the exact correction that changed your answer.

\textbf{Agent 3 $\rightarrow$ Agent 5:} I agree with your final answer \texttt{D}. Your current computation now follows the same complete sequence as the correct solution, so I would not revise it further.

\textbf{Update phase.}

\textbf{Agent 1:} I keep \texttt{D}. The critiques independently verify the same computation and remind me that the subtraction must be from goods available, not from sales. \textit{Confidence: 4}.

\textbf{Agent 2:} I keep \texttt{D}. There is no remaining formula issue: I now compute cost of goods sold from the gross profit margin and use it to estimate ending inventory. \textit{Confidence: 4}.

\textbf{Agent 3:} I keep \texttt{D}. The group agrees with the original cost-of-goods-sold method, and no critique introduces a competing calculation that changes the value. \textit{Confidence: 4}.

\textbf{Agent 4:} I keep \texttt{D}. The option mapping is settled, and the accounting steps still produce $56{,}300$. \textit{Confidence: 4}.

\textbf{Agent 5:} I keep \texttt{D}. The gross profit method is now clear: gross profit determines cost of goods sold, and cost of goods sold determines ending inventory. \textit{Confidence: 4}.

\vspace{1mm}
\textbf{Round-end metrics:}
\begin{center}
\footnotesize
\begin{tabular}{@{}ll@{}}
\toprule
Metric & Value \\
\midrule
Correct agents & 5/5 \\
Group vote & \texttt{D} \\
Answer entropy & 0.00 \\
Mean confidence & 4.00 \\
Changed agents & 0 \\
W2R / R2W & 0 / 0 \\
\bottomrule
\end{tabular}
\end{center}

\end{roundbox}

\paragraph{Why this case demonstrates the effect of PEAR.}
This trajectory shows how PEAR can amplify a minority correct answer without requiring dense all-to-all debate. Initially, only one of five agents selects the correct option, while the group vote is wrong. The first routing step exposes wrong agents to the correct dissenting reasoning and to critiques about option mapping and the gross-profit formula. After one debate round, all agents switch to the correct answer. Later rounds keep the same sparse critique budget but no longer destabilize the consensus: answer entropy remains zero and no right-to-wrong transition occurs.

\paragraph{Diagnostic summary.}
Across the full debate, the example has $\mathrm{W2R}=4$, $\mathrm{R2W}=0$, final answer entropy $0.00$, and a balanced final influence distribution. The key correction happens in Round 1, where the initially minority correct reasoning becomes the group consensus through targeted sparse critique routing.

% \newpage

\section{LLM Usage}
\label{app:ai-usage}

We used large language model assistants for paper polishing (improving prose clarity and grammar) and code refactoring during the development of this work. All scientific claims, experimental results, and analyses were produced and verified by the authors; no AI-generated text was included without author review.

%%%%%%%%%%%%%%%%%%%%%%%%%%%%%%%%%%%%%%%%%%%%%%%%%%%%%%%%%%%%

%\newpage
%\input{sections/Checklist}

\end{document}